\documentclass[11pt,article]{article}
\usepackage[left=2cm,right=2cm]{geometry}

\usepackage{natbib}

\usepackage{hyperref}
\usepackage{url}

\usepackage[utf8]{inputenc} 
\usepackage[T1]{fontenc}    
\usepackage{hyperref}       
\usepackage{url}            
\usepackage{booktabs}       
\usepackage{amsfonts}       
\usepackage{nicefrac}       
\usepackage{microtype}      
\usepackage{amsmath}
\usepackage{algorithm,algorithmic}
\usepackage{graphicx}
\usepackage{amsmath,amssymb,amsthm}
\usepackage{algorithm,algorithmic}
\usepackage{subfigure}

\usepackage{xcolor}
\usepackage{framed}
\usepackage{graphicx}
\usepackage{mathtools}
\graphicspath{ {images/} }
\usepackage{tikz}
\usetikzlibrary{arrows,shapes,snakes,automata,backgrounds,petri}
\usepackage{pdfpages}
\usepackage{hyperref}
\hypersetup{
	colorlinks=true,
	linkcolor=blue,
	filecolor=magenta,      
	urlcolor=cyan,
}

\usepackage[T1]{fontenc}
\usepackage[utf8]{inputenc}
\usepackage{tabularx,ragged2e,booktabs,caption}
\newcolumntype{C}[1]{>{\Centering}m{#1}}

\usepackage{caption}

\usepackage{enumitem}

\usepackage{thmtools}

\usepackage{hyperref}
\usepackage{cleveref}

\usepackage{mathrsfs}
\usepackage{mathtools}

\usepackage{enumitem}
\usepackage{verbatim}
\usepackage{url}
\usepackage{thm-restate}
\usepackage{wrapfig}

\usepackage{bbm}

\usepackage{parskip}

\usepackage{graphbox}

\usepackage{subfigure}

\newtheoremstyle{definition}
{3pt} 
{3pt} 
{} 
{} 
{\bfseries} 
{.} 
{.5em} 
{} 

\theoremstyle{definition}

\newtheorem{theorem}{Theorem}[section]

\newtheorem{assumption}[theorem]{Assumption}

\begin{document}
\title{Biased Gradient Estimate with Drastic Variance Reduction \\ for Meta Reinforcement Learning}
\author{Yunhao Tang\footnote{yt2541@columbia.edu}\\ DeepMind}
 \date{}
 
\maketitle
\begin{abstract}
Despite the empirical success of meta reinforcement learning (meta-RL), there are still a number poorly-understood discrepancies between theory and practice. Critically, biased gradient estimates are almost always implemented in practice, whereas prior theory on meta-RL only establishes convergence under unbiased gradient estimates. In this work, we investigate such a discrepancy. In particular, (1) We show that unbiased gradient estimates have variance $\Theta(N)$ which linearly depends on the sample size $N$ of the inner loop updates; (2) We propose linearized score function (LSF) gradient estimates, which have bias $\mathcal{O}(1/\sqrt{N})$ and variance $\mathcal{O}(1/N)$; (3) We show that most empirical prior work in fact implements variants of the LSF gradient estimates. This implies that practical algorithms "accidentally" introduce bias to achieve better performance; (4) We establish theoretical guarantees for the LSF gradient estimates in meta-RL regarding its convergence to stationary points, showing better dependency on $N$ than prior work when $N$ is large.
\end{abstract}

\section{Introduction}

By design, many reinforcement learning (RL) algorithms learn from scratch. This entails RL to achieve high profile success in a number of important and challenging applications \citep{mnih2013,silver2016,levine2016}. However, at the same time, RL is highly inefficient compared to how humans learn, usually consuming orders of magnitude more samples to acquire skills at the same level as humans.  One potential source of such inefficiencies is that unlike humans, RL algorithms do not exploit prior knowledge on the tasks at hand. 

To resolve such an issue, meta-reinforcement learning (meta-RL) formalizes the learning and transfer of prior knowledge in RL~\citep{duan2016rl,wang2016learning,finn2017model}. On a high level, an agent interacts with a distribution of tasks at \emph{meta-training} time. The objective is that after meta-training, the agent can learn significantly faster when faced with unseen tasks at \emph{meta-testing} time. If an agent achieves good performance at meta-testing time, it  embodies the ability to transfer knowledge from prior experiences during meta-training. There are many concrete formulations of meta-RL (see, e.g. \citep{wang2016,duan2016rl,houthooft2018evolved,rakelly2019efficient,zintgraf2019varibad,fakoor2019meta,ortega2019meta,oh2020discovering,xu2020meta}), Our focus is meta-RL through gradient-based adaptations 
\citep{finn2017model}, where the agent carries out policy gradient (PG) inner loop updates \citep{sutton1999} at both meta-training and meta-testing time. 

\paragraph{Motivation.}
Our work is motivated by a number of important discrepancies between meta-RL theory and practice. Recently, there is a growing interest in establishing performance guarantees for meta-RL algorithms with unbiased gradient estimates \citep{fallah2020provably}. However, since the inception of the field, meta-RL practitioners have almost always implemented biased gradient estimates \citep{finn2017model,al2017continuous,rothfuss2018promp,liu2019taming,tang2021unifying}. 
It is natural to ask: why are unbiased gradient estimates potentially undesirable in practice, and what do we gain by introducing bias into gradient estimates?

\paragraph{Our focus.} We focus on the \emph{N-sample meta-RL objective} where the inner loop updates are $N$-sample PG estimates. In prior work, this was called the E-MAML objective \citep{al2017continuous,rothfuss2018promp,fallah2020provably}, as opposed to the MAML objective \citep{finn2017model} where the inner loop update is exact PG. This objective is of practical interest, because at meta-testing time, inner loop updates can only be implemented with $N$-sample PG estimates. See Sec~\ref{sec:background} for details.

\paragraph{Summary of this work.} We make a number of developments to bridge meta-RL theory and practice.
\begin{itemize}[topsep=0pt,parsep=0pt,partopsep=0pt,leftmargin=*]
	\item \textbf{High variance of unbiased estimates.} By formulating the meta-RL problem as optimizing a generic $N$-sample additive Monte-Carlo objective, we show that the unbiased gradient estimates have variance on the order of $\Theta(N)$, rendering the estimates useless when $N$ is large (see Sec~\ref{sec:mc}).
	\item \textbf{Novel derivation of biased estimates.} We propose the linearized score function (LSF) gradient estimate for the $N$-sample additive Monte-Carlo objective, which has variance  $\mathcal{O}(1/N)$ and bias  $\mathcal{O}(1/\sqrt{N})$. Its application to meta-RL enjoys better properties at large $N$ (see Sec~\ref{sec:linear}). 
    	\item  \textbf{Prior work implements biased estimates.} We observe that despite their claims of unbiasedness, most prior work in fact implements variants of LSF gradient estimates. This implies they are both biased w.r.t. the MAML and the $N$-sample meta-RL objective (see Sec~\ref{sec:prior}).
    \item \textbf{Performance guarantee with better dependency on $N$.} We provide performance guarantee of meta-RL algorithms with biased  estimates. Such guarantee contrasts with results of unbiased estimates, where the guarantee degrades significantly at large $N$   \citep{fallah2020provably} (see Sec~\ref{sec:convergence}). 
\end{itemize}

\section{Background}
\label{sec:background}

\subsection{Task-based Reinforcement Learning}

Consider a Markov decision process (MDP) with state space $\mathcal{S}$ and action space $\mathcal{A}$. At time $t \geq 0$, the agent takes action $a_t\in\mathcal{A}$ in state $s_t\in\mathcal{S}$, receives a reward $r_t$ and transitions to a next state $s_{t+1}\sim p(\cdot|s_t,a_t)$. Without loss of generality, we assume that the at $t=0$ the agent starts at the same state. We assume the reward $r_t=r(s_t,a_t,g)$ to be a deterministic function of state-action pair $(s_t,a_t)$ and the task variable $g \in \mathcal{G}$. The task variable $g\sim p_\mathcal{G}$ is sampled for every episode. A policy $\pi:\mathcal{S}\rightarrow \mathcal{P}(\mathcal{A})$ specifies a distribution over actions at each state. We further assume that the MDP terminates within a finite horizon of $H$ almost surely under all policies.

\paragraph{Parameterized policy.} In general, the policy is parameterized $\pi_\theta$ with parameter $\theta\in\mathbb{R}^D$,

\paragraph{Value function.}
Let $\tau\coloneqq(s_t,a_t,r_t)_{t=0}^{H-1}$ be a trajectory. The policy $\pi_\theta$ induces a distribution over trajectories $p_{\theta,g}(\tau)\coloneqq\Pi_{t=0}^{H-1} p(x_{t+1}|s_t,a_t)\pi_\theta(a_t|s_t,g)$ . We define $R(\tau,g)\coloneqq \sum_{t=0}^{H-1}\gamma^t r(s_t,a_t,g)$ as the cumulative return along trajectory $\tau$ under task $g$. We also define the value function as the expected returns over trajectories $V_g(\pi_\theta)\coloneqq \mathbb{E}_{\tau\sim p_\theta}\left[R(\tau,g)\right]$. We also overload the notations $V_g(\theta)\coloneqq V_g(\pi_\theta)$.

Note that unlike other work in RL, we define the value function as expected cumulative returns starting from the \emph{initial state}, which we assume to be a fixed single state. This definition will greatly simplify notations in later sections.

\paragraph{Policy gradient and stochastic estimates.} Policy gradient (PG) \citep{sutton1999} is the gradient of the value function with respect to policy parameter $\nabla_\theta V_g(\theta)=\mathbb{E}_{\tau\sim p_{\theta,g}}[R(\tau,g)\nabla_\theta \log p_{\theta,g}(\tau)]$. In practice, it is not feasible to compute PG exactly and it is of interest to construct stochastic PG estimates given sampled trajectories. Indeed,  $\hat{\nabla}_\theta V_g(\theta) = R(\tau,g)\nabla_\theta \log p_{\theta,g}(\tau)$ with $\tau\sim p_{\theta,g}$ is an unbiased PG estimate in that $\mathbb{E}[\hat{\nabla}_\theta V_g(\theta)] = \nabla_\theta V_g(\theta)$. 

\subsection{Meta Reinforcement Learning}

Meta-RL aims to maximize the average value function evaluated at the updated policy parameter 
$
    \theta_N'=\theta+\eta \frac{1}{N}\sum_{i=1}^N R(\tau_i,g)\nabla_\theta \log p_{\theta,g}(\tau_i)
$ obtained by an ascent step with $N$-sample PG estimates. This ascent step is also called the \emph{inner loop update}. Here, $(\tau_i)_{i=1}^N \sim p_{\theta,g}$ i.i.d. and $\eta$ is a fixed stepsize. Formally, consider the following optimization problem,
\begin{align}
    \max_\theta\   \mathbb{E}_{g}\left[F_N(\theta,g)\right],  F_N(\theta,g) \coloneqq \mathbb{E}_{(\tau_i)_{i=1}^N}\left[V_g\left(\theta+\eta \frac{1}{N}\sum_{i=1}^N R(\tau_i,g)\nabla_\theta \log p_{\theta,g}(\tau_i)\right)\right],\label{eq:emaml}
\end{align}
The expectations are over the goal distribution $g\sim p_\mathcal{G}$ and random trajectories $(\tau_i)_{i=1}^N\sim p_{\theta,g}$. The $N$-sample PG estimate update from $\theta$ to $\theta_N'$ is called the \emph{inner loop update}. We call $L_N$ the \emph{$N$-sample meta-RL objective} due to its critical dependency on $N$. Since the task distribution $p_{\mathcal{G}}$ does not depend on $\theta$, we mostly focus on discussing of properties of $L_N$ as a function of $\theta$ in later sections. The $N$-sample meta-RL objective was initially proposed in \citep{al2017continuous,rothfuss2018promp} under the name E-MAML and analyzed in \citep{fallah2020provably} in more theoretical contexts.

\paragraph{The limit case $N\rightarrow\infty$.} Under mild conditions,  the limit exists when $N\rightarrow\infty$ and Eqn~\ref{eq:emaml} converges to the following problem
\begin{align}
    \max_\theta \ \mathbb{E}_{g}\left[F_\infty(\theta,g)\right],  F_\infty(\theta,g) \coloneqq V_g\left(\theta+\eta \nabla_\theta V_g(\theta)\right).\label{eq:maml}
\end{align}
In other words, the inner loop update becomes exact PG ascent. This objective was proposed in the initial MAML framework \citep{finn2017model}.

\paragraph{Short notes on prior work.} Though prior literature mainly focuses on deriving gradient estimates to the MAML objective, we show that there is a fundamental challenge in obtaining unbiased estimates (see Sec~\ref{sec:prior}). Instead, we start the discussion in Sec~\ref{sec:mc} on the $N$-sample meta-RL objective.

\subsection{Stochastic Gradient Estimates of Monte-Carlo Objectives}

To facilitate discussions in later sections, we provide a brief background on optimizing general Monte-Carlo objectives. Monte-Carlo (MC) objectives are common in RL,  generative modeling and various probabilistic machine learning problems (see, e.g., \citep{blei2017,mohamed2020monte} for related reviews). In its general form, MC objectives are defined as $L(\theta)\coloneqq\mathbb{E}_{X\sim p_\theta}\left[f(X)\right]$ where random variables $X$ are drawn from a distribution $p_\theta$ whose density is a differentiable function of $\theta$. For simplicity, we first consider when $f$ depends explicitly on $X$ only, though it can also depend on $\theta$, which we will discuss shortly. To optimize $L(\theta)$, it is of direct interest to construct unbiased estimates to $\nabla_\theta L(\theta)$. 
\paragraph{Score function (SF)  gradient estimate.} Assume $f$ is bounded\footnote{Here, we assume $f$ to be bounded for simplicity, though the SF gradient estimate is well defined and unbiased under more general assumptions \citep{mohamed2020monte}. This boundedness assumption is satisfied for the meta-RL application to be discussed later.}. The  SF gradient estimate is  defined as follows
\begin{align*}
    \hat{\nabla}_\theta^{\text{SF}}L(\theta)\coloneqq f(X) \nabla_\theta \log p_\theta(X),X\sim p_\theta.
\end{align*}
By construction, the estimate is  unbiased. However, due to the gradient of score function $\nabla_\theta \log p_\theta(X)$, the estimate often has high variance in practice.

\paragraph{Path-wise (PW) gradient estimate.} If $X$ is reparameterizable (RE), 
there exists an elementary distribution $\zeta\sim p_\zeta$ (e.g. gaussian  $\mathcal{N}(0,1)$) and  a function $\mathcal{T}_\theta$ such that $ \mathcal{T}_\theta(\zeta)$ is equal in distribution to $X$. When $X$ is RE and $f$ is differentiable, the PW gradient estimate is defined as 
\begin{align*}
    \hat{\nabla}_\theta^{\text{PW}}L(\theta)\coloneqq\left[\nabla_X f(X) \right]_{X=\mathcal{T}_\theta(\zeta)}\nabla_\theta \mathcal{T}_\theta(\zeta) ,\zeta\sim p_\zeta.
\end{align*}
The PW gradient estimate is also unbiased. Intuitively, since PW gradient estimate makes use of the gradient $\nabla_X f(X)$, it enjoys lower variance compared to the SF gradient estimate in many applications \citep{kingma2013auto}. However,  the PW gradient estimate is less generally applicable due to assumptions on $X$ and $f$. For example, those assumptions are not satisfied for important applications such as RL and meta-RL. 

\section{Meta-RL as $N$-sample Additive Monte-Carlo Objective}
\label{sec:mc}

We start our discussion by extending the MC objective to $N$-sample additive MC objective. This general framework encompasses meta-RL as a special case. It also allows us to naturally derive a novel estimate with significant variance reduction in Sec~\ref{sec:linear}. 

\subsection{$N$-sample Additive Monte-Carlo Objective}
Let $(X_i)_{i=1}^N\sim p_\theta$ be i.i.d. samples from a parameterized distribution $p_\theta$ on domain $\mathcal{X}$. Define $\phi:\mathcal{X}\mapsto \mathbb{R}^h$ as feature mapping function that takes $x\in\mathcal{X}$ as input and outputs a $h$-dimensional feature $\phi(x)$. Let $f:\mathbb{R}^h\mapsto\mathbb{R}$ be a scalar function that maps from the feature space to a scalar value. We define the $N$-sample additive MC objective as follows,
\begin{align}
  L(\theta) \coloneqq  \mathbb{E}_{(X_i)_{i=1}^N} \left[f\left(\frac{\sum_{i=1}^N \phi(X_i)}{N}\right)\right].\label{eq:mc-n-sample}
\end{align}
The $N$-sample additive MC objective can be recovered as a special case of the MC objective by defining $X\coloneqq (X_i)_{i=1}^N$. However, we will find it useful to make clear how the property of $L(\theta)$ explicitly depends on $N$. Though the objective defines interactions between $\phi(X_i)$ in an additive manner,  we will see that this seemingly restrictive definition generalizes the $N$-sample meta-RL objective $F_N(\theta,g)$ as a special case. In the following, we ground the discussion with a toy example.

\paragraph{Toy $N$-sample Additive MC Objective.} Consider when $p_\theta$ is a parameterized Gaussian distribution $\mathcal{N}(\mu,\sigma^2)$ where $\sigma>0$ is fixed. The feature mapping $\phi$ and objective $f$ are both identity functions.

\subsection{Gradient Estimates for $N$-sample Additive MC Objective}

The SF gradient estimate to the $N$-sample additive MC objective is
\begin{align}
    \hat{\nabla}_\theta^{\text{SF}} L(\theta) \coloneqq f\left(\frac{\sum_{i=1}^N \phi(X_i)}{N}\right) \sum_{i=1}^N \nabla_\theta \log p_\theta (X_i), (X_i)_{i=1}^N \sim p_\theta.\label{eq:mc-n-sample-sf}
\end{align}
Since the SF gradient estimate changes distributions over $N$ variables at the same time, $\sum_{i=1}^N \nabla_\theta \log p_\theta (X_i)$ sums over $N$ terms. This implies high variance, which we calculate exactly for the toy example\footnote{Throughout the presentation, we will use the "Big O" notations. See Appendix~\ref{appendix:bigo} for their detailed definitions.}.

\begin{restatable}{lemma}{lemmaexamplemcvariance}\label{lemma:example-mc} In the toy $N$-sample MC objective example,  $\mathbb{V}\left[\hat{\nabla}_\theta^{\text{SF}}L(\theta)\right]=\Theta(N)$.
\end{restatable}
The variance depends linearly on $N$! This makes the estimate very hard to use in applications with large $N$. Compared to the SF gradient estimate, when the PW gradient estimate  $\hat{\nabla}_\theta^{\text{PW}}L(\theta)$ is available, it has much lower variance. In the toy example, it is indeed the case since $X=\sigma\cdot\zeta+\mu,\zeta\sim \mathcal{N}(0,1)$,
\begin{restatable}{lemma}{lemmaexamplemcvariancepw}\label{lemma:example-pw-variance} In the toy $N$-sample MC objective example,  $\mathbb{V}\left[\hat{\nabla}_\theta^{\text{PW}}L(\theta)\right]=0$.
\end{restatable}

The zero variance is specialized to the toy example, though in general the PW gradient estimate also tends to achieve very small variance, making it a golden standard for  unbiased gradient estimates. However, PW gradient estimates are not generally applicable, e.g., to RL and meta-RL objectives.

\subsection{Gradient Estimates for Generalized $N$-sample Additive MC Objective}

Next, we the discussion to the case where $f,\phi$ depends on parameter $\theta$. Define the \emph{generalized} $N$-sample additive MC objective as follows
\begin{align}
  G(\theta) \coloneqq  \mathbb{E}_{(X_i)_{i=1}^N} \left[f\left(\frac{\sum_{i=1}^N \phi(X_i,\theta)}{N},\theta\right)\right].\label{eq:general-mc-n-sample}
\end{align}
We start by deriving exact gradient to the objective
\begin{restatable}{lemma}{lemmageneralmcnsampleexact}\label{lemma:general-mc-n-sample-exact} Let $\bar{\phi}_N\coloneqq\frac{1}{N}\sum_{i=1}^N \phi(X_i,\theta)$. The generalized $N$-sample additive MC objective has gradient $\nabla_\theta G(\theta)$ as follows, where $(X_i)_{i=1}^N\sim p_{\theta}$ i.i.d., 
\begin{align*}
    \mathbb{E}_{(X_i)_{i=1}^N}\left[ \underbrace{f\left(\bar{\phi}_N,\theta\right) \sum_{i=1}^N \nabla_\theta \log p_\theta (X_i)}_{\text{term (i)}} +  \underbrace{ \nabla_\theta f\left(\theta,\bar{\phi}_N\right)+\left(\frac{1}{N}\sum_{i=1}^N \nabla_\theta \phi(\theta,X_i)\right) \nabla_{\bar{\phi}_N}f(\theta,\bar{\phi}_N)}_{\text{term (ii)}}\right].
\end{align*}
\end{restatable}

\paragraph{Generalized SF gradient estimate.}  With access to samples $(X_i)_{i=1}\sim p_\theta$, we define the generalized SF gradient estimate  $\hat{\nabla}^{\text{SF}}_\theta G(\theta)$ as follows
\begin{align}
    \underbrace{f\left(\bar{\phi}_N,\theta\right) \sum_{i=1}^N \nabla_\theta \log p_\theta (X_i)}_{\text{term (i)}} +  \underbrace{\nabla_\theta f\left(\theta,\bar{\phi}_N\right)+\left(\frac{1}{N}\sum_{i=1}^N \nabla_\theta \phi(X_i,\theta)\right) \nabla_{\bar{\phi}_N}f(\bar{\phi}_N,\theta)}_{\text{term (ii)}}.   \label{eq:general-mc-n-sample-sf}
\end{align}

The two terms in the estimate echo the two terms in the exact gradient in Lemma~\ref{lemma:general-mc-n-sample-exact}. Term (i) corresponds to the SF gradient estimate in Eqn~\ref{eq:mc-n-sample-sf}. Term (ii) is a direct result of how $f,\phi$ depends on $\theta$. We provide a full derivation in Appendix~\ref{appendix:derive}.
Examining term (i) and term (ii), we argue that the variance of the overall estimate mainly comes from term (i). This is because term (ii) \textbf{averages} over $N$ terms (e.g., with $\bar{\phi}_N$) whereas term (i) \textbf{sums} over $N$ score function gradients $\nabla_\theta \log p_\theta(X_i)$. 

\subsection{Meta-RL as Generalized $N$-sample Additive MC Objective}

With the conversion: $X_i\coloneqq \tau_i, \phi(X_i,\theta)\coloneqq R(\tau_i,g)\nabla_\theta \log p_{\theta,g}(\tau_i)$ and $f(\bar{\phi}_N,\theta)=V_g(\theta+\eta \bar{\phi}_N)$, we cast meta-RL as a special instance of the generalized $N$-sample additive MC objective. We compute gradient of the $N$-sample objective $J_N(\theta,g)\coloneqq \nabla_\theta F_N(\theta,g)$ as a direct result of Lemma~\ref{lemma:general-mc-n-sample-exact}. 

\begin{restatable}{lemma}{lemmaemamlgradient}\label{lemma:emaml-gradient} Let  $\theta_N'\coloneqq \theta+\eta \frac{1}{N}\sum_{i=1}^N R(\tau_i,g)\nabla_\theta \log p_{\theta,g}(\tau_i)$ be the (random) updated parameter. Let $\tau_i\sim p_{\theta,g}$ i.i.d. and let $\nabla V_g(\theta_N')$ denotes $\left[\nabla_\theta V_g(\theta) \right]_{\theta=\theta_N'}$. Then  $J_N(\theta,g)\coloneqq \nabla_\theta F_N(\theta,g)$ is 
\begin{align}
     \underbrace{\mathbb{E}_{(\tau_i)_{i=1}^N} \left[  V_g(\theta_N') \sum_{i=1}^N \nabla_\theta \log p_{\theta,g}(\tau_i)\right]}_{\eqqcolon J_N^{(i)}(\theta,g)} + \underbrace{\mathbb{E}_{(\tau_i)_{i=1}^N}\left[\left(I+\eta \frac{1}{N}\sum_{i=1}^N R(\tau_i,g) \nabla_\theta^2 \log p_{\theta,g}(\tau_i)\right) \nabla V_g(\theta_N') \right]}_{\eqqcolon J_N^{(ii)}(\theta,g)}.\label{eq:emaml-grad}
\end{align}
\end{restatable}

We reiterate intuitions about the two gradient terms in the context of meta-RL.
The  parameter $\theta$ influences the objective $F_N(\theta,g)$ in two different ways. 
The first term arises from the fact that the $N$ random trajectories are sampled from $p_{\theta,g}$, which depends on $\theta$. 
The second term is a result of how $\theta$ impacts 
$F_N(\theta,g)$ explicitly through the inner loop $N$-sample PG estimate.

\paragraph{Unbiased meta-RL gradient estimate.} In the following, we specify an algorithmic procedure to construct unbiased estimates to $J_N(\theta,g)$. This is a direct instantiation of the generalized SF gradient estimate in Eqn~\ref{eq:general-mc-n-sample-sf} in the context of meta-RL. 
\begin{restatable}{corollary}{coroemamlgradient}\label{coro:emaml-gradient} First, sample $(\tau_i)_{i=1}^N\sim p_{\theta,g}$ and compute the updated parameter $\theta_N'$. Then, construct unbiased estimates to $\nabla V_g(\theta_N')$ and $V_g(\theta_N')$, e.g. with trajectories sampled under $\pi_{\theta_N'}$. Let these estimates be $\nabla \hat{V}_g(\theta_N')$ and $\hat{V}_g(\theta_N')$ respectively\footnote{For now, we just require the estimates to be unbiased. In Sec~\ref{sec:convergence}, we make these estimates concrete for refined convergence analysis.}. The final estimate is
\begin{align}
   \underbrace{ \hat{V}_g(\theta_N') \sum_{i=1}^N \nabla_\theta \log p_{\theta,g}(\tau_i)}_{\eqqcolon \hat{J}_{N,\text{SF}}^{(i)}(\theta,g)} +  \underbrace{\left(I+\eta \frac{1}{N}\sum_{i=1}^N R(\tau_i,g) \nabla_\theta^2 \log p_{\theta,g}(\tau_i)\right) \nabla \hat{V}_g(\theta_N') }_{\eqqcolon \hat{J}_{N,\text{SF}}^{(ii)}(\theta,g)}.\label{eq:unbiased-emaml}
\end{align}
Both terms are unbiased $\mathbb{E}[\hat{J}_{N,\text{SF}}^{(i)}(\theta,g)]=J_N^{(i)}(\theta,g), \mathbb{E}[\hat{J}_{N,\text{SF}}^{(ii)}(\theta,g)]=J_N^{(ii)}(\theta,g)$ with respect to the two terms in Eqn~\ref{eq:emaml-grad}. This implies that the overall estimate is also unbiased. 
\end{restatable}

\paragraph{Variance of  unbiased gradient estimate.}

As direct implications of the properties of SF gradient estimate and generalized SF gradient estimate, $\hat{J}_N$ has very high variance. In fact, building on the $N$-sample additive MC objective toy example, we can construct meta-RL examples where unbiased estimates have variance $\Theta(N)$. See Appendix~\ref{appendix:toy-metarl} for more details. Our objective now is to develop new estimates which bypass the high variance of the unbiased estimate.

\section{Linearized Score Function Gradient Estimate}
\label{sec:linear}
We now introduce a major development in this paper: a new gradient estimate for the $N$-sample additive MC objective. This estimate is in general biased but has significantly lower variance ($\mathcal{O}(1/N)$) compared to the SF gradient estimate ($\Theta(N)$), making it attractive in practice when $N$ is large.

\subsection{Linearized Score Function Gradient Estimate for $N$-sample additive MC Objective.}  When the PW gradient estimate is applicable, it often has lower variance than the SF gradient estimate. Previously, we argue that this is because PW leverages gradient information in the objective $f$ while SF does not. Building on this intuition, we propose a new gradient estimate called \emph{linearized score function} (LSF) gradient estimate as follows,
\begin{align}
    \hat{\nabla}_\theta^{\text{LSF}}L(\theta)\coloneqq  \frac{1}{N}\sum_{i=1}^N \left[\nabla f\left(\bar{\phi}_N\right)\right]^T \phi(X_i) \nabla_\theta \log p_\theta (X_i), \ \text{where} \ \nabla f(\bar{\phi}_N)\coloneqq [\nabla_x f(x)]_{x=\bar{\phi}_N}.\label{eq:mc-n-sample-lsf}
\end{align}

\paragraph{Derivation of the estimate.}
The naming \emph{linearized} comes from how the estimate was derived in the first place, which we show in detail in Appendix~\ref{appendix:derive}. We highlight a number of important steps and insights below. In a nutshell, LSF is derived by linearizing the term $f(\bar{\phi}_N)$ used in the SF gradient estimate, based on Taylor expansion near the reference point $\bar{\phi}\coloneqq\mathbb{E}[\phi(X)]$. We decompose $f(\bar{\phi}_N)$ into three terms below,
\begin{align*}
    f(\bar{\phi}_N)=\underbrace{f(\bar{\phi})}_{\text{constant term}} +\underbrace{ \left[\nabla f(\bar{\phi})\right]^T \left(\bar{\phi}_N - \bar{\phi}\right)}_{\text{linear term}} + \underbrace{\mathcal{O}\left( \left\lVert \bar{\phi}_N - \bar{\phi}\right\rVert_2^2\right)}_{\text{residual term}}.
\end{align*} 
When multiplying the above by $\sum_{i=1}^N \nabla_\theta \log p_\theta(X_i)$, we obtain the SF gradient estimate on the LHS, the RHS is
\begin{align*}
   \underbrace{f(\bar{\phi}) \sum_{i=1}^N \nabla_\theta \log p_\theta(X_i)}_{\text{constant\ term}} +\underbrace{ \left[\nabla f(\bar{\phi})\right]^T \left(\bar{\phi}_N - \bar{\phi}\right) \sum_{i=1}^N \nabla_\theta \log p_\theta(X_i)}_{\text{linear\ term}} + \underbrace{\mathcal{O}\left( \left\lVert \bar{\phi}_N - \bar{\phi}\right\rVert_2^2\right) \sum_{i=1}^N \nabla_\theta \log p_\theta(X_i)}_{\text{residual\ term}}.
\end{align*} 
The LSF gradient estimate stems from the linear term above. We now justify removing the other terms: we remove the residual term because it should vanish as $N\rightarrow\infty$ and $\bar{\phi}_N\rightarrow\bar{\phi}$. In fact, since the residual term is not zero in expectation, the LSF gradient estimate is biased. However, its bias is of order $\mathcal{O}(1/\sqrt{N})$, as we will see later.

Secondly, we also remove the constant term it has mean zero and only contributes variance. In this sense, the LSF gradient estimate can be understood as being related to a SF gradient estimate with baseline (or control variate) $\left(f(\bar{\phi}_N)-f(\bar{\phi})\right)\sum_{i=1}^N \nabla_\theta \log p_\theta(X_i)$ \citep{ross6277simulation}. When $N$ is large and $f(\bar{\phi}_N)\approx f(\bar{\phi})$, we should expect the control variate to achieve significant variance reduction in practice. 

Finally, we use $\nabla f(\bar{\phi}_N)$ as an proxy to $\nabla f(\phi_N)$. Overall, we see that the LSF gradient estimate makes use of the gradient of $f$ yet does not require reparameterization of the random variables $X_i$. In this sense, it is more general than the PW gradient estimate, yet leverages more information than the SF gradient estimate. 

\paragraph{Properties of the estimate.} The following result shows LSF achieves significant variance reduction.
\begin{restatable}{lemma}{lemmaexamplemcvariancelsf}\label{lemma:example-lsf-variance} In the toy $N$-sample MC objective example,  $\mathbb{V}\left[\hat{\nabla}_\theta^{\text{LSF}}L(\theta)\right]=\mathcal{O}(1/N)$.
\end{restatable}

In the toy example, the PW gradient estimate is the gold standard unbiased estimate with zero variance. Yet, as discussed before, it is not generally applicable. The LSF gradient estimate has variance $\mathcal{O}(1/N)$, which decays as $N$ increases. This makes LSF applicable in large $N$ regimes. However, unlike the SF gradient estimate which is by design unbiased, the LSF gradient estimate is in general biased. Nevertheless, when applying the LSF gradient estimate to the $N$-sample meta-RL objective, we can characterize the bias to be of order $\mathcal{O}(1/N)$ (see Proposition~\ref{prop:bias3}).

\subsection{Gradient Estimate for Generalized $N$-sample Additive MC Objective}
We extend the LSF gradient estimate to the generalized $N$-sample additive MC objective in Eqn~\ref{eq:general-mc-n-sample}. We do so by replacing the term (i) SF gradient estimate by LSF gradient estimate in Eqn~\ref{eq:general-mc-n-sample-sf}. This produces the generalized LSF gradient estimate $\hat{\nabla}_\theta^{\text{LSF}}G(\theta)$ as follows,
\begin{align}
    \underbrace{ \frac{1}{N}\sum_{i=1}^N \left[\nabla_{\bar{\phi}_N} f\left(\bar{\phi}_N,\theta\right)\right]^T  \phi(X_i,\theta) \nabla_\theta \log p_\theta (X_i,\theta)}_{\text{term (i)}} +  \underbrace{\nabla_\theta f\left(\theta,\bar{\phi}_N\right)+\left(\frac{1}{N}\sum_{i=1}^N \nabla_\theta \phi(X_i,\theta)\right) \nabla_{\bar{\phi}_N}f(\bar{\phi}_N,\theta)}_{\text{term (ii)}}.   \label{eq:general-mc-n-sample-lsf}
\end{align}
Due to the bias in the LSF gradient estimate, the generalized LSF gradient estimate is also biased. However, the key trade-off is that the new term (i) in Eqn~\ref{eq:general-mc-n-sample-lsf} \textbf{averages} over $N$ samples and achieves significantly smaller variance than the generalized SF gradient estimate. 

\subsection{Biased Gradient Estimate to Meta-RL Objective} We next apply the generalized LSF gradient estimate to the $N$-sample meta-RL objective.
\begin{restatable}{corollary}{coromamlgradient}\label{coro:lst-maml} Let  $u_i\coloneqq \nabla_\theta \log p_{\theta,g}(\tau_i)$. The generalized LSF gradient estimate $\hat{J}_{N,\text{LSF}}(\theta,g)$ to $F_N(\theta,g)$ can be expressed as follows, 
\begin{align*}
      \left(\eta \frac{1}{N}\sum_{i=1}^N R(\tau_i,g) u_i u_i^T\right)\nabla V_g(\theta_N')  \left(I+\eta \frac{1}{N}\sum_{i=1}^N R(\tau_i,g) \nabla_\theta^2 \log p_{\theta,g}(\tau_i)\right) \nabla V_g(\theta_N') ,
\end{align*}
Define $\nabla \hat{V}_g,\hat{V}_g$ as unbiased estimates to $\nabla V_g,V_g$. The following estimate has the same expectation as $\hat{J}_{N,\text{LSF}}(\theta,g)$ defined above,
\begin{align}
     \underbrace{ \left(\eta \frac{1}{N}\sum_{i=1}^N R(\tau_i,g) u_i u_i^T\right)\nabla \hat{V}_g(\theta_N')  }_{\eqqcolon \hat{J}_{N,\text{LSF}}^{(i)}(\theta,g)}+\underbrace{\left(I+\eta \frac{1}{N}\sum_{i=1}^N R(\tau_i,g) \nabla_\theta^2 \log p_{\theta,g}(\tau_i)\right) \nabla \hat{V}_g(\theta_N') }_{\eqqcolon \hat{J}_{N,\text{LSF}}^{(ii)}(\theta,g)},\label{eq:biased-emaml}
\end{align}
\end{restatable}

Since Eqn~\ref{eq:biased-emaml} can be computed via a similar procedure as  Lemma~\ref{coro:emaml-gradient}, we call it the generalized LSF gradient estimate to the meta-RL objective. 
While the unbiased SF gradient estimate $\hat{J}_{N,\text{SF}}^{(ii)}$ has high variance when $N$ is large, the LSF gradient estimate $\hat{J}_{N,\text{LSF}}^{(i)}$ achieves a good trade-off between bias and variance. We will show how such trade-off impacts the convergence analysis in Sec~\ref{sec:convergence}.

\paragraph{Connections to Hessian estimation.} We can rewrite the LSF gradient estimate in Eqn~\ref{eq:biased-emaml} as $(I+\eta \hat{H}_N(\theta) \nabla ) \hat{V}_g(\theta_N')$, where $\hat{H}_N(\theta)=\frac{1}{N}\sum_{i=1}^N R(\tau_i,g)\left( u_i u_i^T + \nabla_\theta^2 \log p_{\theta,g}(\tau_i) \right)$. It turns out that $I+\eta \hat{H}_N(\theta)$ is an unbiased estimate to the Hessian matrix $\mathbb{E}[\hat{H}_N(\theta)]=\nabla^2 V_g(\theta)$. Most prior work focus on variance reduction for estimating this function \citep{foerster2018dice,mao2019baseline,rothfuss2018promp,tang2021unifying}.

\paragraph{Connections to exact gradient for meta-RL objective $F_\infty(\theta,g)$.} It is now worthwhile to contrast the generalized LSF gradient estimate to the gradient of $J_\infty(\theta,g)\coloneqq \nabla_\theta F_\infty(\theta,g)$.
\begin{restatable}{corollary}{coromamlgradientequiv}\label{coro:maml-gradient-equiv}Let $u_i\coloneqq \nabla_\theta \log p_{\theta,g}(\tau_i)$ and $\theta'=\theta+\eta\mathbb{E}_{\tau\sim p_{\theta,g}}[R(\tau,g)\nabla_\theta \log p_{\theta,g}(\tau)]$ be the updated parameter with exact PG ascent. In the following, let $(\tau_i)_{i=1}^N\sim p_{\theta,g}$ i.i.d., then $J_{\infty}(\theta,g)$ is 
\begin{align}
     \underbrace{\mathbb{E}_{(\tau_i)_{i=1}^N}\left[\eta \frac{1}{N}\sum_{i=1}^N R(\tau_i,g) u_i u_i^T \nabla V_g(\theta')\right]}_{\eqqcolon J_{\infty}^{(i)}(\theta,g)} + \underbrace{\mathbb{E}_{(\tau_i)_{i=1}^N}\left[\left(I+\eta \frac{1}{N}\sum_{i=1}^N R(\tau_i,g) \nabla_\theta^2 \log p_{\theta,g}(\tau_i)\right) \nabla V_g(\theta') \right]}_{\eqqcolon J_{\infty}^{(ii)}(\theta,g)} ,
    \label{eq:maml-gradient-equiv}
\end{align}
\end{restatable}

 Here, since $\theta'$ is the updated parameter resulting from exact PG ascent, it is not easy to construct unbiased estimate to $J_\infty(\theta,g)$. This is because even if we can compute $\theta_N'$ as $N$-sample unbiased estimate to $\theta'$, in general we still have $\nabla V_g(\theta')\neq \mathbb{E}[\nabla V_g(\theta_N')]$. However, note that there are similarities between the parametric forms of $\hat{J}_{N,\text{LSF}}(\theta,g)$ and $J_\infty(\theta,g)$. We can interpret $\hat{J}_{N,\text{LSF}}(\theta,g)$ as also a biased estimate to $\hat{J}_{N,\text{LSF}}(\theta,g)$, obtained by replacing $\theta'$ with $\theta_N'$.

\section{Discussion on prior work}
\label{sec:prior}

\paragraph{$N$-sample meta-RL objective.} As noted earlier, the $N$-sample meta-RL objective was considered in both empirical  \citep{al2017continuous,rothfuss2018promp} and theoretical contexts \citep{fallah2020provably}. This objective is of practical interest because of budget on inner loop samples. The limit case $N=\infty$ was considered in the original MAML formulation of meta-RL \citep{finn2017model}.

\paragraph{Unbiased gradient to the limit case $J_\infty(\theta,g)$.} In the author's original implementation of the MAML gradient estimate with auto-differentiation libraries \citep{finn2017model}, a term equivalent to $J_\infty^{(i)}(\theta,g)$ was unintentionally dropped, resulting in a biased estimate. This fuels the motivation for a number of follow-up work to derive unbiased gradients \citep{foerster2018dice,liu2019taming}. However, they are \textbf{biased} in general. This is mainly because practical algorithms can only estimate $\nabla_g V_g(\theta_N')$ instead of $\nabla_g V_g(\theta')$, while the latter is required to estimate $J_\infty(\theta,g)$ in an unbiased way. This observation was also hinted at recently in \citep{tang2021unifying}.

\paragraph{Prior work in fact constructs the LSF gradient estimate.}
Since most prior work derive meta-RL gradient estimates based on $J_\infty(\theta,g)$ \citep{foerster2018dice,rothfuss2018promp,liu2019taming,tang2021unifying}, and due to the \emph{accidental} replacement of $\theta'$ by $\theta_N'$, we conclude that they in fact construct variants of the LSF gradient estimate (see comments following Corollary~\ref{coro:maml-gradient-equiv}). In particular, they construct $\hat{J}$ such that $\mathbb{E}[\hat{J}]=\mathbb{E}[\hat{J}_{N,\text{LSF}}(\theta,g)]$ but with potentially lower variance. All of them focus on reducing variance of $\hat{H}_N(\theta)$. Though in theory reducing the variance of $\hat{H}_N(\theta)$ does not necessarily guarantee improvements, in practice, this seems to be very critical. Variance reduction methods include control variates \citep{liu2019taming}, as well as introducing further bias to the estimate of $\hat{H}_N(\theta)$  \citep{rothfuss2018promp,tang2021unifying}.

\paragraph{Unbiased gradient estimate to $N$-sample meta-RL objective.}

The exact gradient and unbiased gradient estimate to $N$-sample meta-RL objective was derived in \citep{al2017continuous,rothfuss2018promp,fallah2020provably}. A comprehensive derivation was carried out in \citep{rothfuss2018promp}, where they contrasted $J_\infty(\theta,g)$ with $J_N(\theta,g)$. However, they claimed that $J_\infty^{(ii)}(\theta,g)=J_N^{(ii)}(\theta,g)$, which is \textbf{not} true. Our derivation shows that $J_\infty^{(ii)}(\theta,g)\neq J_N^{(ii)}(\theta,g)$ in general because $\mathbb{E}[\nabla V_g(\theta_N')]\neq \nabla V_g(\theta')$. 

\paragraph{Convergence analysis of gradient-based meta-learning and meta-RL.}

Due to the highly complex objective landscape of meta learning, most theoretical analysis focuses on convergence to stationary points. 
Recently, \citep{fallah2020convergence} established generic convergence guarantees for gradient-based meta-learning algorithms for supervised learning with one inner loop update. Recently, \citep{ji2020multi}  extended the analysis to multi-step inner loop updates. For meta-RL, \citep{fallah2020provably} established convergence for the $N$-sample meta-RL objective. They motivated the objective in a similar manner as \citep{al2017continuous,rothfuss2018promp} and constructed unbiased estimates exactly as the generalized SF gradient estimate $\hat{J}_{N,\text{SF}}(\theta,g)$. However, since the estimate has variance linear in $N$, the final guarantee becomes less applicable in practice. Contrast to this work, we show how the biased generalized LSF gradient estimate achieves performance guarantee with more desirable dependency on $N$.

\section{ Convergence Analysis with Linearized Score Function Gradient Estimate}
\label{sec:convergence}

We start by presenting the meta-RL full algorithm with generalized LSF gradient estimate. Note that the pseudocode also closely relates to how practical algorithms are implemented (see more in Appendix~\ref{appendix:exp}).

\subsection{Full Algorithm and Key Assumptions}

The full meta-RL algorithm is in Algorithm 1. There are two important notes on the details: 
(1) We instantiate the unbiased gradient estimate $\nabla V_g(\theta_N')$ by $M$-sample PG estimates with trajectories collected under the updated parameter $\theta_N'$; (2) So far we have focused on presenting gradient estimate for a single task $g$. In practice, we sample a batch of $B$ tasks $(g_i)_{i=1}^B$ and compute gradient estimate for each $\hat{J}_{N,\text{LSF}}(\theta,g_i)$. The overall gradient $\hat{J}_{N,\text{LSF}}$ is averaged across tasks, which is then used for the final update $\theta_{t+1}=\theta_t+\alpha\hat{J}_{N,\text{LSF}} $ at each iteration.

\begin{algorithm}[h]
\label{algo:evaluation-subroutine}
\begin{algorithmic}
\REQUIRE \textbf{Inputs}: Hyper-parameters: batch sizes $(B,N,M)$. Step size $\eta$. Initial parameter $\theta_1=\theta$. 
\FOR{$t=1,2...$}
\STATE \textbf{Inner loop sampling}. Sample $B$ task variables $g_i$ and $N$ trajectories under $(\tau_{i,j})_{j=1}^N\sim p_{\theta,g_i}$.
\STATE \textbf{Inner update.} Compute inner loop update $\theta_{i,N}'=\theta_t+\eta\frac{1}{N}\sum_{j=1}^N R(\tau_{i,j},g_i)\nabla_\theta \log p_{\theta,g_i}(\tau_{i,j})$.
\STATE \textbf{Outer sampling at adapted parameters.} Collect $M$ trajectories $(\tau_{i,k}')_{k=1}^M \sim p_{\theta_{i,N}',g_i}$ for the outer loop PG estimate $\nabla_\theta \hat{V}_{g_i}(\theta_{i,N}')=\frac{1}{M}\sum_{k=1}^M R(\tau_{i,k}',g_i)\nabla_\theta \log p_{\theta,g_i}(\tau_{i,k}')$.
\STATE \textbf{Gradient estimate and update.} Compute  $\hat{J}_{N,\text{LSF}}(\theta,g_i)$ based on Eqn~\ref{eq:biased-emaml}. Then compute $\hat{J}_{N,\text{LSF}}=\frac{1}{B}\sum_{i=1}^B \hat{J}_{N,\text{LSF}}(\theta,g_i)$ as the full estimate. Update outer loop $\theta_{t+1}=\theta_t+\alpha\hat{J}_{N,\text{LSF}}$.
\ENDFOR
\STATE Output trained meta-RL policy $\pi_\theta$.
\caption{$N$-sample meta-RL algorithm with linearized SF gradient estimate}
\end{algorithmic}
\end{algorithm}

We need a few common assumptions \citep{fallah2020provably} for theoretical analysis. 

\begin{assumption}
\label{assum:common} 
(Smooth parameterization assumptions) For all $s\in\mathcal{S},a\in\mathcal{A},g\in\mathcal{G}$ and $\theta\in\mathbb{R}^D$, $\left\lVert \nabla_\theta \log \pi_\theta(a|s,g) \right\rVert_2\leq G_1$ and $\left\lVert \nabla^2_\theta \log \pi_\theta(a|s,g) \right\rVert_2\leq G_2$\footnote{See Appendix~\ref{appendix:norms} for definitions of tensor norms and variance.}.
\end{assumption}

In addition, we impose a smoothness condition on the value function. This could be converted into an equivalent assumption on the parameterization.
\begin{assumption}
\label{assum:common2} 
(Smooth value function assumption) For all $g\in\mathcal{G},\theta\in\mathbb{R}^D$, $\left\lVert \nabla^3 V_g(\theta) \right\rVert_2 \leq L$.
\end{assumption}

All the above assumptions can be conveniently verified for e.g., tabular MDP (finite $\mathcal{X},\mathcal{A}$ and $\mathcal{G}$) with soft-max parameterization of the policy, where $\pi_\theta(a|s,g)\propto \exp(\theta(s,a,g))$ with parameter $\theta=\{\theta(s,a,g)\}$.

\subsection{Performance Guarantee}

We now provide performance guarantees of the LSF gradient estimate. It is worth noting that since we are interested in the dependency on $N$, the analysis does not necessarily obtain the optimal dependency on other problem parameters (such as the parameter dimension $D$ or horizon $H$). We leave potential improvements to future work.

The meta-RL objective takes an average over the  parameter-independent distribution and hence its gradient $J_N(\theta)\coloneqq \mathbb{E}_g[J_N(\theta,g)]$.  As previously discussed, since $\mathbb{E}[\hat{J}_{N,\text{LSF}}(\theta,g)]\neq J_N(\theta,g)$, the generalized LSF gradient estimate $\hat{J}_{N,\text{LSF}}(\theta)$ is biased in general. We start by characterizing its bias against $J_N(\theta)$. Our results below characterize the dependency of various quantities on $N$, and folding other constants into $\mathcal{O}(1)$. See Appendix~\ref{appendix:proof} for concrete dependencies on other constants in our analysis.

\begin{restatable}{proposition}{propbiasthird}\label{prop:bias3} For all $\theta\in\mathbb{R}^D$,
$
     \left\lVert \mathbb{E}[\hat{J}_{N,\text{LSF}}(\theta)]-J_N(\theta)\right\rVert_2  = \mathcal{O}(1/\sqrt{N})$.
\end{restatable}
The bias is benign as it vanishes when $N$ is large.
We next characterize the variance of the estimate.
\begin{restatable}{proposition}{propvariance} \label{prop:variance}
For all $\theta\in\mathbb{R}^D$, $\mathbb{V}\left[\hat{J}_{N,\text{LSF}}(\theta)\right]=\underbrace{ \mathcal{O}(1/M)+\mathcal{O}(1/B)}_{\mathcal{O}(1)}+\mathcal{O}(1/N)$.
\end{restatable}
The bound $\mathcal{O}(1/M)+\mathcal{O}(1/B)$ means to show the dependency on the sample size $B$ and $M$. When they quantities do not depend on $N$, they are considered $\mathcal{O}(1)$. The three terms on the upper bound above indicate sources of randomness that contribute the variance of the generalized LSF gradient estimate $\hat{J}_{N,\text{LSF}}(\theta)$: the batch of $B$ tasks, the batch of $N$ inner loop trajectories $\tau_{ij}$ per task and the batch of $M$ trajectories $\tau_{ik}'$ for estimating outer loop PG. 

The bound is in general $\mathcal{O}(1)$ when $N$ is large. This is because in general it is not possible to get rid of the variance induced by a finite $B$ and $M$. However, wen we let $B,M\rightarrow\infty$, the total variance is of order $\mathcal{O}(1/N)$. This is consistent with the variance of the LSF gradient estimate for the $N$-sample MC objective (see Lemma~\ref{lemma:example-lsf-variance}). Now we are ready present the convergence guarantee of Algorithm 1. We show its convergence to a stationary point of the objective $J_N(\theta)$.

\begin{restatable}{theorem}{propmain}\label{prop:main} With a properly chosen learning rate in Algorithm 1, for any $\epsilon>0$, with $T_\text{LSF}=2\max\{\frac{1}{\epsilon^2+\mathcal{O}(1/N)} , \frac{\mathcal{O}(1)+\mathcal{O}(1/N)}{\epsilon^4 + \mathcal{O}(1/N^2)}\}$ iterations of the algorithm,  we have 
\begin{align*}
    \min_{1\leq t\leq T_\text{LSF}} \mathbb{E}[\left\lVert J_N(\theta_t) \right\rVert_2^2]=  \epsilon^2 + \mathcal{O}(1/N)\eqqcolon \delta_\text{LSF}.
\end{align*}
\end{restatable}
It is insightful to contrast with the result of 
\citep{fallah2020provably}, where they they analyze the generalized SF gradient estimate. They show $T_\text{SF}=\mathcal{O}(1)\frac{1}{\alpha}\min\{\epsilon^{-2},\Theta(N^{-2})\}$ and $\delta_\text{SF}=\epsilon^2 + \Theta(N^3 \alpha)$, where recall that $\alpha$ is the learning rate.

\begin{figure}[t]
    \centering
    \subfigure[Bias-variance trade-off ]{\includegraphics[keepaspectratio,width=.24\textwidth]{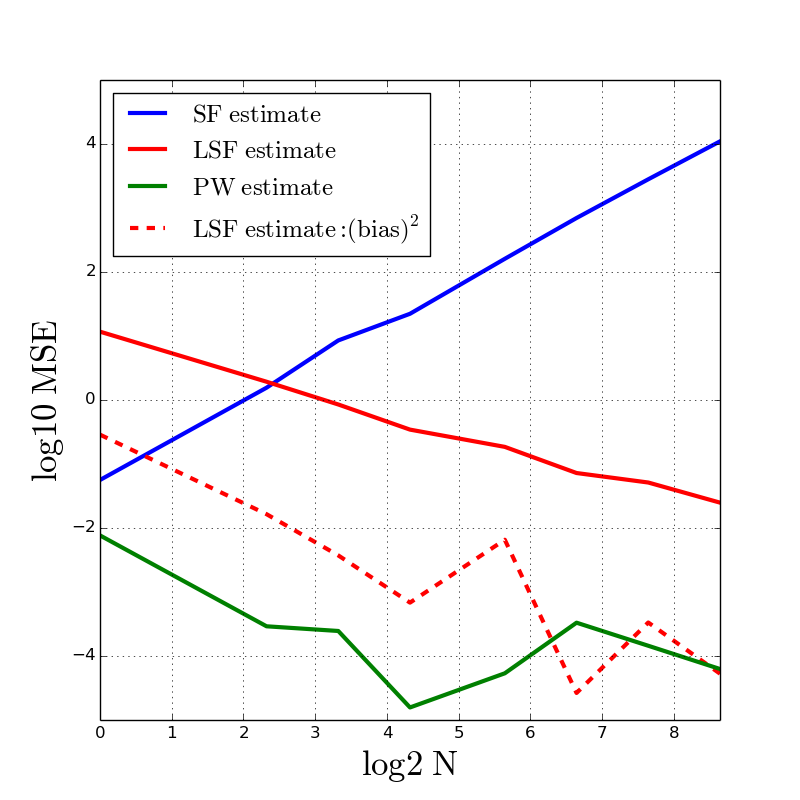}}
    \subfigure[1-D Optimization ]{\includegraphics[keepaspectratio,width=.24\textwidth]{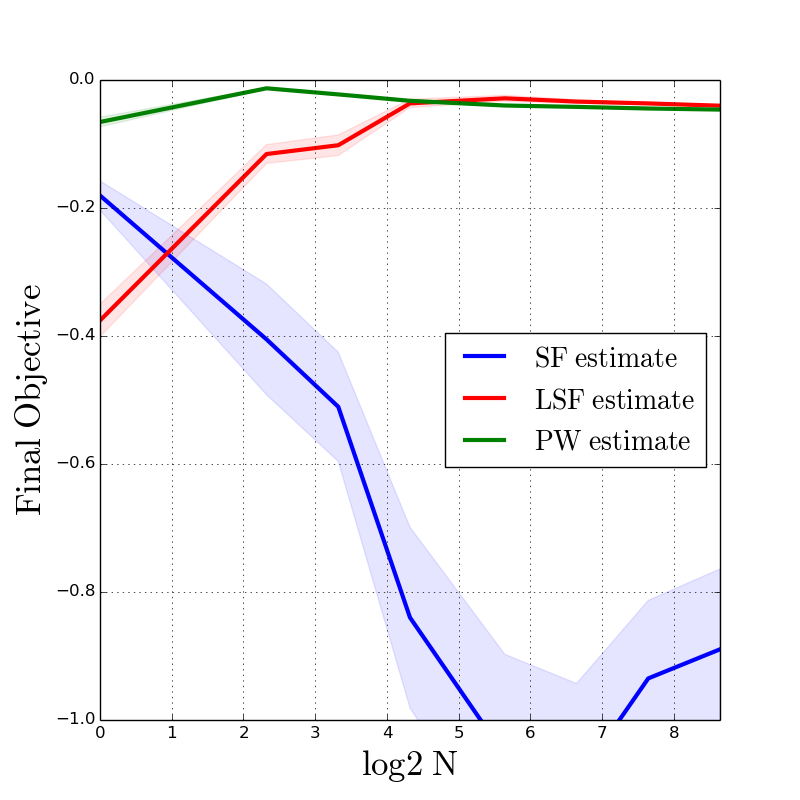}}
    \subfigure[Meta-RL HalfCheetah ]{\includegraphics[keepaspectratio,width=.24\textwidth]{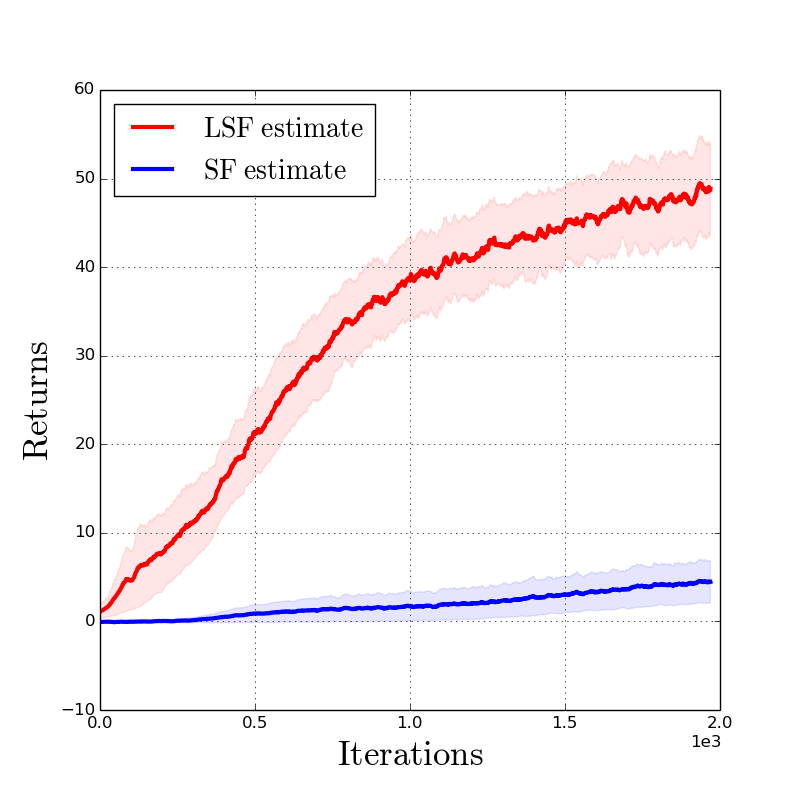}}
    \subfigure[Meta-RL Walker2D ]{\includegraphics[keepaspectratio,width=.24\textwidth]{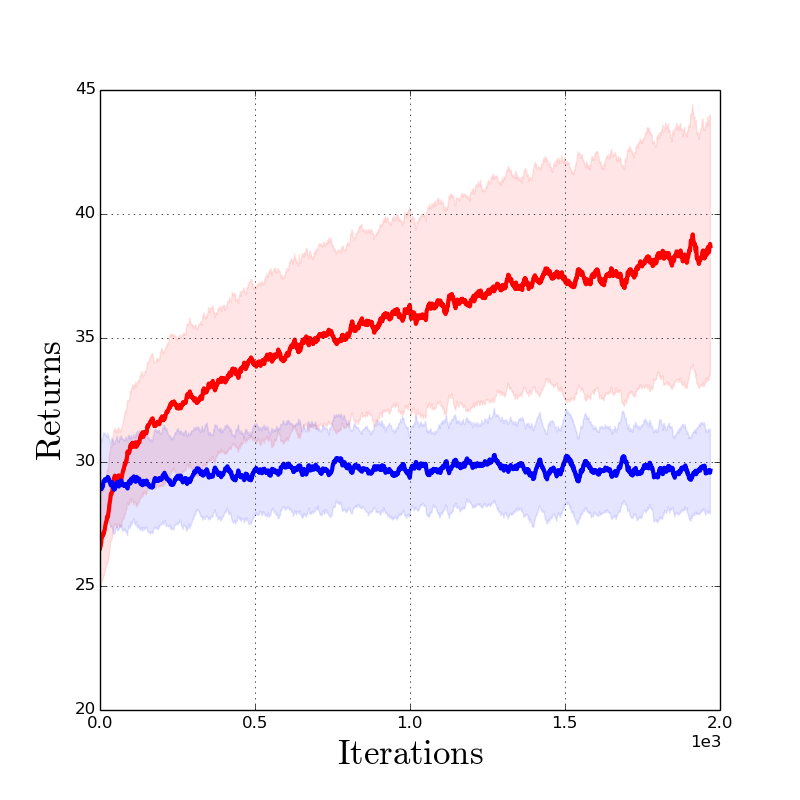}}
    \caption{(a) 1-D problem bias and variance trade-off: The MSE is computed at the initial update.
    At small $N$, LSF has higher MSE than 
    SF. However, the MSE of SF increases with $N$ due to variance;   (b) 1-D problem learning curves:  LSF outperforms SF. PW is the best when available. Each curve averages over $100$ runs; (c)-(d) High-dimensional meta-RL problems: LSF outperforms SF consistently across all tasks. Each curve averages over $5$ runs. See Appendix~\ref{appendix:exp} for more results.}
    \label{fig:exp}
\end{figure}

To see that the generalized LSF gradient estimate achieves a better dependency on $N$ than the generalized SF gradient estimate, we fix $\epsilon$ and $N$, and adjust the learning rate $\alpha$ of the SF gradient estimate. First, we require the asymptotic error to have the same dependency on $N$ by setting $\alpha=1/N^4$, in which case $\delta_\text{SF}=\epsilon^2+\Theta(1/N)$ while $\delta_\text{LSF}=\epsilon^2+\mathcal{O}(1/N)$. This implies $T_\text{SF}=\mathcal{O}(1)N^4 \min\{\epsilon^{-2},\Theta(N^{-2})\}$, which is significantly worse than $T_\text{LSF}$ when $N$ is large. Intuitively, this is because the generalized SF gradient estimate has much higher variance, which requires a very small $\alpha$ 
to achieve the same level of asymptotic error as the LSF gradient estimate. As a result, this takes the algorithm many more iterations to converge.

Equivalently, we can require both estimates to converge with the same number of iterations. Assuming $\epsilon$ is small enough such that $N$ is the dominating factor in the asymptotic error, this means we should choose $\alpha=1/N$. In this case, we have $\delta_\text{SF}=\epsilon^2 + \Theta(N^2)$, which is significantly worse than $\delta_\text{LSF}=\epsilon^2+\mathcal{O}(1/N)$. This is because in order to converge with the same number of iterations, SF gradient estimate requires a relatively large learning rate. Large learning rates prevent the algorithm from dissipating the high variance in the gradient estimate, which ultimately lead to high asymptotic errors.

\paragraph{Convergence to stationary point of $F_\infty(\theta)$.} Since $\left\lVert J_N(\theta)-J_\infty(\theta)\right\rVert_2\leq \mathcal{O}(1/N)$ (Proposition~\ref{prop:bias1}), the above result also implies that with $T_\text{LSF}$, the algorithm also obtains the stationary point of $J_\infty(\theta)$ up to error $\epsilon^2 + \mathcal{O}(1/N)$. As such, we can interpret the optimization of $F_N(\theta)$ as a proxy to $F_\infty(\theta)$ when $N$ is large.

\section{Experiments}

We carry out experiments to illustrate theoretical insights. We briefly discuss the results, see Appendix~\ref{appendix:exp} for further details on the experiments.  Below, we refer to the generalized SF (LSF) estimate also as the SF (LSF) estimate when the contexts are clear.

\paragraph{1-D optimization problem.} We maximize $L(\theta)=\mathbb{E}_{(X_i)_{i=1}^N\sim p_\theta}[-(\bar{X}_N-1)^2]$ where $\bar{X}_N$ is the average over $N$ samples and $p_\theta=\mathcal{N}(\theta,1^2)$. Fig~\ref{fig:exp}(a) shows the bias-variance trade-off of different gradient estimates. 
The mean-square error (MSE) of SF increases with $N$ due to variance (because SF is unbiased), whereas the MSE of LSF decreases with $N$ due to both decreasing variance and bias. Consistent with Lemma~\ref{lemma:example-pw-variance} and many empirical observations in prior work, PW 
achieves the lowest MSE due to low variance. Fig~\ref{fig:exp}(b) compares the final objective after a fixed number of ascent updates. It shows as $N$ increases, SF degrades significantly in performance whereas LSF improves. The performance of different estimates correlate strongly with the MSE: when $N$ is small, LSF is outperformed by SF due to the bias; when $N$ is large, LSF catches up with PW while the performance of SF degrades significantly. These observations are also consistent with discussions in Sec~\ref{sec:convergence}.

\paragraph{High-dimensional meta-RL problems.} We contrast the SF and LSF gradient estimates in meta-RL. To implement both SF and LSF, we use the Markov structure to reduce variance compared to the original "trajectory-based" estimates. This turns out to be quite critical in practice. Fig~\ref{fig:exp}(b)-(d) shows that the LSF gradient estimate outperforms the SF gradient estimate, where the high variance of the unbiased SF gradient estimate consistently hinders learning across all tasks we tested on.

Recall that the LSF gradient estimate effectively constructs an unbiased estimate to $\nabla^2 V_g(\theta)$.
As a side note, we find that biased estimates to $\nabla^2 V_g(\theta)$  \citep{finn2017model,rothfuss2018promp,tang2020taylor} leads to generally better performance. We leave the study of such observations to future work.

\section{Conclusion}
By formulating the $N$-sample meta-RL objective as a special case of $N$-sample additive MC objective, we identify the high variance ($\Theta(N)$) of naive SF gradient estimate. The LSF gradient estimate, which is biased but has low variance ($\mathcal{O}(1/N)$), achieves theoretical guarantees with much more benign dependency on $N$. As a result, our analysis suggests the necessity of employing biased gradient estimates in practice.

Meanwhile, we also make the observation that many prior work turned out to implement variants of the LSF gradient estimate. This implies that despite their claim of unbiasedness, the practical gradient estimates are almost always biased. This is consistent with the empirical results: when $N$ is large, unbiased estimates are doomed to be unstable due to high variance, while properly biased estimates perform much better.

Overall, we believe our results help better understand the subtle design choices in meta-RL practice, and might entail the design of new algorithms in future work. 

\section{Acknowledgements}
Many thanks to Mark Rowland, Mohammad Gheshlaghi Azar, Bernardo Ávila Pires and Tom Zahavy for discussing the idea. A special thank you to Botao Hao for reviewing a draft of this work.

\newpage

\bibliographystyle{iclr2019_conference}
\bibliography{your_bib_file}

\begin{thebibliography}{35}
\providecommand{\natexlab}[1]{#1}
\providecommand{\url}[1]{\texttt{#1}}
\expandafter\ifx\csname urlstyle\endcsname\relax
  \providecommand{\doi}[1]{doi: #1}\else
  \providecommand{\doi}{doi: \begingroup \urlstyle{rm}\Url}\fi

\bibitem[Ajalloeian \& Stich(2020)Ajalloeian and Stich]{ajalloeian2020analysis}
Ahmad Ajalloeian and Sebastian~U Stich.
\newblock Analysis of sgd with biased gradient estimators.
\newblock \emph{arXiv preprint arXiv:2008.00051}, 2020.

\bibitem[Al-Shedivat et~al.(2017)Al-Shedivat, Bansal, Burda, Sutskever,
  Mordatch, and Abbeel]{al2017continuous}
Maruan Al-Shedivat, Trapit Bansal, Yuri Burda, Ilya Sutskever, Igor Mordatch,
  and Pieter Abbeel.
\newblock Continuous adaptation via meta-learning in nonstationary and
  competitive environments.
\newblock \emph{arXiv preprint arXiv:1710.03641}, 2017.

\bibitem[Blei et~al.(2017)Blei, Kucukelbir, and McAuliffe]{blei2017}
David~M Blei, Alp Kucukelbir, and Jon~D McAuliffe.
\newblock Variational inference: A review for statisticians.
\newblock \emph{Journal of the American Statistical Association}, 112\penalty0
  (518):\penalty0 859--877, 2017.

\bibitem[Duan et~al.(2016)Duan, Schulman, Chen, Bartlett, Sutskever, and
  Abbeel]{duan2016rl}
Yan Duan, John Schulman, Xi~Chen, Peter~L Bartlett, Ilya Sutskever, and Pieter
  Abbeel.
\newblock Rl2: Fast reinforcement learning via slow reinforcement learning.
\newblock \emph{arXiv preprint arXiv:1611.02779}, 2016.

\bibitem[Fakoor et~al.(2019)Fakoor, Chaudhari, Soatto, and
  Smola]{fakoor2019meta}
Rasool Fakoor, Pratik Chaudhari, Stefano Soatto, and Alexander~J Smola.
\newblock Meta-q-learning.
\newblock \emph{arXiv preprint arXiv:1910.00125}, 2019.

\bibitem[Fallah et~al.(2020{\natexlab{a}})Fallah, Georgiev, Mokhtari, and
  Ozdaglar]{fallah2020provably}
Alireza Fallah, Kristian Georgiev, Aryan Mokhtari, and Asuman Ozdaglar.
\newblock Provably convergent policy gradient methods for model-agnostic
  meta-reinforcement learning.
\newblock \emph{arXiv preprint arXiv:2002.05135}, 2020{\natexlab{a}}.

\bibitem[Fallah et~al.(2020{\natexlab{b}})Fallah, Mokhtari, and
  Ozdaglar]{fallah2020convergence}
Alireza Fallah, Aryan Mokhtari, and Asuman Ozdaglar.
\newblock On the convergence theory of gradient-based model-agnostic
  meta-learning algorithms.
\newblock In \emph{International Conference on Artificial Intelligence and
  Statistics}, pp.\  1082--1092. PMLR, 2020{\natexlab{b}}.

\bibitem[Farquhar et~al.(2019)Farquhar, Whiteson, and
  Foerster]{farquhar2019loaded}
Gregory Farquhar, Shimon Whiteson, and Jakob Foerster.
\newblock Loaded dice: Trading off bias and variance in any-order score
  function gradient estimators for reinforcement learning.
\newblock 2019.

\bibitem[Finn et~al.(2017)Finn, Abbeel, and Levine]{finn2017model}
Chelsea Finn, Pieter Abbeel, and Sergey Levine.
\newblock Model-agnostic meta-learning for fast adaptation of deep networks.
\newblock In \emph{International Conference on Machine Learning}, pp.\
  1126--1135. PMLR, 2017.

\bibitem[Foerster et~al.(2018)Foerster, Farquhar, Al-Shedivat, Rockt{\"a}schel,
  Xing, and Whiteson]{foerster2018dice}
Jakob Foerster, Gregory Farquhar, Maruan Al-Shedivat, Tim Rockt{\"a}schel, Eric
  Xing, and Shimon Whiteson.
\newblock Dice: The infinitely differentiable monte carlo estimator.
\newblock In \emph{International Conference on Machine Learning}, pp.\
  1529--1538. PMLR, 2018.

\bibitem[Houthooft et~al.(2018)Houthooft, Chen, Isola, Stadie, Wolski, Ho, and
  Abbeel]{houthooft2018evolved}
Rein Houthooft, Richard~Y Chen, Phillip Isola, Bradly~C Stadie, Filip Wolski,
  Jonathan Ho, and Pieter Abbeel.
\newblock Evolved policy gradients.
\newblock \emph{arXiv preprint arXiv:1802.04821}, 2018.

\bibitem[Ji et~al.(2020)Ji, Yang, and Liang]{ji2020multi}
Kaiyi Ji, Junjie Yang, and Yingbin Liang.
\newblock Multi-step model-agnostic meta-learning: Convergence and improved
  algorithms.
\newblock \emph{arXiv preprint arXiv:2002.07836}, 2020.

\bibitem[Kingma \& Ba(2014)Kingma and Ba]{kingma2014adam}
Diederik~P Kingma and Jimmy Ba.
\newblock Adam: A method for stochastic optimization.
\newblock \emph{arXiv preprint arXiv:1412.6980}, 2014.

\bibitem[Kingma \& Welling(2013)Kingma and Welling]{kingma2013auto}
Diederik~P Kingma and Max Welling.
\newblock Auto-encoding variational bayes.
\newblock \emph{arXiv preprint arXiv:1312.6114}, 2013.

\bibitem[Levine et~al.(2016)Levine, Finn, Darrell, and Abbeel]{levine2016}
Sergey Levine, Chelsea Finn, Trevor Darrell, and Pieter Abbeel.
\newblock End-to-end training of deep visuomotor policies.
\newblock \emph{The Journal of Machine Learning Research}, 17\penalty0
  (1):\penalty0 1334--1373, 2016.

\bibitem[Liu et~al.(2019)Liu, Socher, and Xiong]{liu2019taming}
Hao Liu, Richard Socher, and Caiming Xiong.
\newblock Taming maml: Efficient unbiased meta-reinforcement learning.
\newblock In \emph{International Conference on Machine Learning}, pp.\
  4061--4071. PMLR, 2019.

\bibitem[Mao et~al.(2019)Mao, Foerster, Rockt{\"a}schel, Al-Shedivat, Farquhar,
  and Whiteson]{mao2019baseline}
Jingkai Mao, Jakob Foerster, Tim Rockt{\"a}schel, Maruan Al-Shedivat, Gregory
  Farquhar, and Shimon Whiteson.
\newblock A baseline for any order gradient estimation in stochastic
  computation graphs.
\newblock In \emph{International Conference on Machine Learning}, pp.\
  4343--4351. PMLR, 2019.

\bibitem[Mnih et~al.(2013)Mnih, Kavukcuoglu, Silver, Graves, Antonoglou,
  Wierstra, and Riedmiller]{mnih2013}
Volodymyr Mnih, Koray Kavukcuoglu, David Silver, Alex Graves, Ioannis
  Antonoglou, Daan Wierstra, and Martin Riedmiller.
\newblock Playing atari with deep reinforcement learning.
\newblock \emph{arXiv preprint arXiv:1312.5602}, 2013.

\bibitem[Mohamed et~al.(2020)Mohamed, Rosca, Figurnov, and
  Mnih]{mohamed2020monte}
Shakir Mohamed, Mihaela Rosca, Michael Figurnov, and Andriy Mnih.
\newblock Monte carlo gradient estimation in machine learning.
\newblock \emph{J. Mach. Learn. Res.}, 21\penalty0 (132):\penalty0 1--62, 2020.

\bibitem[Oh et~al.(2020)Oh, Hessel, Czarnecki, Xu, van Hasselt, Singh, and
  Silver]{oh2020discovering}
Junhyuk Oh, Matteo Hessel, Wojciech~M Czarnecki, Zhongwen Xu, Hado van Hasselt,
  Satinder Singh, and David Silver.
\newblock Discovering reinforcement learning algorithms.
\newblock \emph{arXiv preprint arXiv:2007.08794}, 2020.

\bibitem[Ortega et~al.(2019)Ortega, Wang, Rowland, Genewein, Kurth-Nelson,
  Pascanu, Heess, Veness, Pritzel, Sprechmann, et~al.]{ortega2019meta}
Pedro~A Ortega, Jane~X Wang, Mark Rowland, Tim Genewein, Zeb Kurth-Nelson,
  Razvan Pascanu, Nicolas Heess, Joel Veness, Alex Pritzel, Pablo Sprechmann,
  et~al.
\newblock Meta-learning of sequential strategies.
\newblock \emph{arXiv preprint arXiv:1905.03030}, 2019.

\bibitem[Popoviciu(1965)]{popoviciu1965certaines}
Tiberiu Popoviciu.
\newblock Sur certaines in{\'e}galit{\'e}s qui caract{\'e}risent les fonctions
  convexes.
\newblock \emph{Analele Stiintifice Univ.“Al. I. Cuza”, Iasi, Sectia Mat},
  11:\penalty0 155--164, 1965.

\bibitem[Rakelly et~al.(2019)Rakelly, Zhou, Finn, Levine, and
  Quillen]{rakelly2019efficient}
Kate Rakelly, Aurick Zhou, Chelsea Finn, Sergey Levine, and Deirdre Quillen.
\newblock Efficient off-policy meta-reinforcement learning via probabilistic
  context variables.
\newblock In \emph{International conference on machine learning}, pp.\
  5331--5340. PMLR, 2019.

\bibitem[Ross(2002)]{ross6277simulation}
SM~Ross.
\newblock Simulation, 2002, 2002.

\bibitem[Rothfuss et~al.(2018)Rothfuss, Lee, Clavera, Asfour, and
  Abbeel]{rothfuss2018promp}
Jonas Rothfuss, Dennis Lee, Ignasi Clavera, Tamim Asfour, and Pieter Abbeel.
\newblock Promp: Proximal meta-policy search.
\newblock \emph{arXiv preprint arXiv:1810.06784}, 2018.

\bibitem[Schulman et~al.(2015)Schulman, Levine, Abbeel, Jordan, and
  Moritz]{schulman2015}
John Schulman, Sergey Levine, Pieter Abbeel, Michael Jordan, and Philipp
  Moritz.
\newblock Trust region policy optimization.
\newblock In \emph{International Conference on Machine Learning}, pp.\
  1889--1897, 2015.

\bibitem[Silver et~al.(2016)Silver, Huang, Maddison, Guez, Sifre, Van
  Den~Driessche, Schrittwieser, Antonoglou, Panneershelvam, Lanctot,
  et~al.]{silver2016}
David Silver, Aja Huang, Chris~J Maddison, Arthur Guez, Laurent Sifre, George
  Van Den~Driessche, Julian Schrittwieser, Ioannis Antonoglou, Veda
  Panneershelvam, Marc Lanctot, et~al.
\newblock Mastering the game of go with deep neural networks and tree search.
\newblock \emph{nature}, 529\penalty0 (7587):\penalty0 484--489, 2016.

\bibitem[Sutton et~al.(2000)Sutton, McAllester, Singh, and Mansour]{sutton1999}
Richard~S Sutton, David~A McAllester, Satinder~P Singh, and Yishay Mansour.
\newblock Policy gradient methods for reinforcement learning with function
  approximation.
\newblock In \emph{Advances in neural information processing systems}, pp.\
  1057--1063, 2000.

\bibitem[Tang et~al.(2020)Tang, Valko, and Munos]{tang2020taylor}
Yunhao Tang, Michal Valko, and R{\'e}mi Munos.
\newblock Taylor expansion policy optimization.
\newblock \emph{arXiv preprint arXiv:2003.06259}, 2020.

\bibitem[Tang et~al.(2021)Tang, Kozuno, Rowland, Munos, and
  Valko]{tang2021unifying}
Yunhao Tang, Tadashi Kozuno, Mark Rowland, R{\'e}mi Munos, and Michal Valko.
\newblock Unifying gradient estimators for meta-reinforcement learning via
  off-policy evaluation.
\newblock \emph{arXiv preprint arXiv:2106.13125}, 2021.

\bibitem[Todorov et~al.(2012)Todorov, Erez, and Tassa]{todorov2012}
Emanuel Todorov, Tom Erez, and Yuval Tassa.
\newblock Mujoco: A physics engine for model-based control.
\newblock In \emph{Intelligent Robots and Systems (IROS), 2012 IEEE/RSJ
  International Conference on}, pp.\  5026--5033. IEEE, 2012.

\bibitem[Wang et~al.(2016)Wang, Kurth-Nelson, Tirumala, Soyer, Leibo, Munos,
  Blundell, Kumaran, and Botvinick]{wang2016learning}
Jane~X Wang, Zeb Kurth-Nelson, Dhruva Tirumala, Hubert Soyer, Joel~Z Leibo,
  Remi Munos, Charles Blundell, Dharshan Kumaran, and Matt Botvinick.
\newblock Learning to reinforcement learn.
\newblock \emph{arXiv preprint arXiv:1611.05763}, 2016.

\bibitem[Wang et~al.(2015)Wang, Schaul, Hessel, Van~Hasselt, Lanctot, and
  De~Freitas]{wang2016}
Ziyu Wang, Tom Schaul, Matteo Hessel, Hado Van~Hasselt, Marc Lanctot, and Nando
  De~Freitas.
\newblock Dueling network architectures for deep reinforcement learning.
\newblock \emph{arXiv preprint arXiv:1511.06581}, 2015.

\bibitem[Xu et~al.(2020)Xu, van Hasselt, Hessel, Oh, Singh, and
  Silver]{xu2020meta}
Zhongwen Xu, Hado van Hasselt, Matteo Hessel, Junhyuk Oh, Satinder Singh, and
  David Silver.
\newblock Meta-gradient reinforcement learning with an objective discovered
  online.
\newblock \emph{arXiv preprint arXiv:2007.08433}, 2020.

\bibitem[Zintgraf et~al.(2019)Zintgraf, Shiarlis, Igl, Schulze, Gal, Hofmann,
  and Whiteson]{zintgraf2019varibad}
Luisa Zintgraf, Kyriacos Shiarlis, Maximilian Igl, Sebastian Schulze, Yarin
  Gal, Katja Hofmann, and Shimon Whiteson.
\newblock Varibad: A very good method for bayes-adaptive deep rl via
  meta-learning.
\newblock \emph{arXiv preprint arXiv:1910.08348}, 2019.

\end{thebibliography}

\newpage
\appendix

\section{Definitions of "Big O" Notations}
\label{appendix:bigo}

For non-negative functions $f,g$, the notation $f(N)=\mathcal{O}\left(g(N)\right)$ implies that when $N$ is large enough, there exists constant $C\geq 0$ such that $f(N)\leq Cg(N)$.

The notation $f(N)=\Theta\left(g(N)\right)$ implies that $N$ is large enough, there exists constant $C\geq c\geq 0$ such that $cg(N)\leq f(N)\leq Cg(N)$.

\section{Derivation of the Linearized Score Function Gradient Estimate}
\label{appendix:derive}

Since $X_i$s are i.i.d., we expect the average $\bar{\phi}_N=\frac{1}{N}\sum_{i=1}^N \phi(X_i)$ to approach $\bar{\phi}\coloneqq \mathbb{E}\left[\phi(X_i)\right]$ as $N\rightarrow\infty$. This provides a direct motivation to consider the behavior of $f(\bar{\phi}_N)$ near the constant $\bar{\phi}$. In particular, consider the Taylor expansion of $f\left(\bar{\phi}_N\right)$ with $\bar{\phi}$ as its reference point. We decompose $f\left(\bar{\phi}_N\right)$ into three parts,
\begin{align*}
    f\left(\bar{\phi}_N\right) = 
   \underbrace{ f\left(\bar{\phi}\right)}_{\text{constant term}} + \underbrace{\left[\nabla f\left(\bar{\phi}\right)\right]^T \left[  \bar{\phi}_N - \bar{\phi}  \right]}_{\text{linear term}} + \underbrace{\mathcal{O}\left(\left\lVert \bar{\phi}-\bar{\phi}_N \right\rVert_2^2\right)}_{\text{residual  term}}
\end{align*}
If we multiply the above with $\sum_{i=1}^N \nabla_\theta \log p_\theta(X_i)$, we recover the original SF gradient estimate on the LHS. Examining, the LHS, if we drop the residual term of the Taylor expansion, this yields a new estimate,
\begin{align*}
   \underbrace{ f\left(\bar{\phi}\right)\sum_{i=1}^N \nabla_\theta \log p_\theta(X_i) }_{\text{constant term}}  +  \underbrace{\left[\nabla f\left(\bar{\phi}\right)\right]^T \left[ \bar{\phi}_N - \bar{\phi}  \right] \left(\sum_{i=1}^N \nabla_\theta \log p_\theta(X_i)\right)}_{\text{linear term}}.
\end{align*}
We will see that removing the residual term leads to a bias of order $\mathcal{O}(1/\sqrt{N})$ under some mild conditions on $f$. 

Note now that the constant term has mean zero, but is nonzero in general. This implies that this term contributes a large portion of the total variance of this new estimate. It is therefore tempting to remove this term from the estimate. In fact, removing the constant term is equivalent to augmenting the original estimate with a baseline (or control variate) $\left(f(\bar{\phi}_N) - f(\bar{\phi})\right)\sum_{i=1}^N \nabla \log p_\theta(X_i)$ \citep{ross6277simulation}. We should expect the control variate to achieve significant variance reduction when $N$ is large.

We now are left with the linear term alone as the new estimate. Note that if we count each $\nabla_\theta \log p_\theta(X_i)$ as a single term, there are a total of $N^2$ terms in the linear term. This is because we can write
\begin{align*}
    \text{linear term} = \frac{1}{N} \sum_{i=1}^N\sum_{j=1}^N \left[\nabla f(\bar{\phi})\right]^T  \left(\phi(X_j)-\bar{\phi}\right)  \nabla_\theta \log p(X_j).
\end{align*}
It is worth noting that the "off-diagonal" terms have mean zero. In other words, \begin{align*}
    \mathbb{E}\left[\left[\nabla f(\bar{\phi})\right]^T \left(\phi(X_i)-\bar{\phi}\right)  \nabla_\theta \log p(X_j)\right]=0, \ \ i\neq j.
\end{align*} 
We can remove all such terms, reducing the computations to only $N$ "diagonal terms". We expect this removal step to reduce variance in practice as well, because it reduces the number of summations from $N^2$ to $N$.
This produces a new estimate,
\begin{align*}
   \frac{1}{N}\sum_{i=1}^N  \left[\nabla f(\bar{\phi})\right]^T  \left(\phi(X_i)-\bar{\phi}\right) \nabla_\theta \log p(X_i).
\end{align*}
There are two places where $\bar{\phi}$ appears. Since $\bar{\phi}$ is not known in practice, we make two modifications to the above estimate: (1) We replace $f(\bar{\phi})$ by $f(\bar{\phi}_N)$. This will introduce further bias into the estimate; (2)
We remove the second $\bar{\phi}$ for simplicity, as it will not impact the dependency of final estimates' bias and variance on $N$. Importantly, note that since $\bar{\phi}\nabla_\theta \log p(X_i)$ has mean zero, it does not change the mean of the overall estimate. However, since $\bar{\phi}$ can be seen as a control variate, its removal can potentially increase the variance of the estimate. Combining the above modifications, we arrive at the LSF gradient estimate,
\begin{align*}
   \frac{1}{N}\sum_{i=1}^N \left[\nabla f(\bar{\phi}_N)\right]^T  \phi(X_i)\nabla_\theta \log p(X_i).
\end{align*}

\section{Toy Example for Meta-RL}
\label{appendix:toy-metarl}

We construct a toy meta-RL example that illustrates the variance property of the SF gradient estimate vs. LSF gradient estimate. We consider a MDP where the value function $V_g(\theta)\coloneqq V_0$ is a constant for all $\theta$. We also set the adaptation step size $\eta=0$. We also assume a single starting state $x_0$, and that the trajectory terminates after $H=1$ step. As a result, the trajectory consists of a single action $\tau_i\equiv a_i$. We can assume a 1-D action space $\mathcal{A}=\mathbb{R}$, and the policy $\pi(a|\theta,g)=\mathcal{N}(\theta,\sigma^2)$ is a Gaussian distribution with learnable mean $\theta$ and fixed $\sigma$. In this case, the effective objective is $L(\theta,g)=\mathbb{E}_{(a_i)_{i=1}^N \sim \mathcal{N}(\theta,\sigma^2)}[V_g(\theta)]$.

\paragraph{SF gradient estimate.} In this case, the estimate is
$
    V_g(\theta)\sum_{i=1}^N \nabla_\theta \log p_{\theta,g}(\tau_i) = V_0\sum_{i=1}^N \nabla_\theta \log p_{\theta,g}(\tau_i)
$.
The SF gradient estimate bears close resemblance to the SF gradient estimate in the $N$-sample MC estimate. Its variance is of order $\Theta(N)$.

\paragraph{LSF gradient estimate.} 
In this case, we can show that the LSF gradient estimate is effectively the PG estimate at the adapted parameter $\hat{\nabla} V_g(\theta')$ where $\theta'=\theta$. It then depends on how we construct the PG estimate. For example, since $\theta'=\theta$, we can use the $N$ samples generated under $\theta$ to estimate the PG estimate. It then naturally follows that the variance is of order $\mathcal{O}(1/N)$. 

\section{Notations for Norms and Useful Inequalities}
\label{appendix:norms}
For any tensor (vector or matrix) $X$, we define its $2$-norm as 
\begin{align*}
    \left\lVert X \right\rVert_2 \coloneqq \sqrt{\sum_i X_i^2},
\end{align*}
where $i$ sums over components of $X$. The variance is defined as the sum of the variance of its components,
\begin{align*}
    \mathbb{V}[X]\coloneqq \sum_i \mathbb{V}\left[X_i\right].
\end{align*}

We now introduce a number of useful inequalities, which will be heavily used in the proof section.
\paragraph{Operator norm and $2$-norm for matrix.} The oprator norm of a matrix $X$ is defined as: $\left\lVert X\right\rVert_{\text{op},2}\coloneqq \max_{\left\lVert u\right\rVert_2=1} \left\lVert Xu\right\rVert_2$. It is known that $\left\lVert X\right\rVert_{\text{op},2} \leq \left\lVert X\right\rVert_{2}$.

\paragraph{Exchange of norm and expectation.} For any random tensor $X$,
$
    \left\lVert \mathbb{E}\left[X\right] \right\rVert_2 \leq   \mathbb{E}\left[\left\lVert X \right\rVert_2\right]
$. The proof is based on Jensen's inequality and the fact that $2$-norm is a convex function of its argument.

\paragraph{Cauchy–Schwarz (CS) inequality for random variables.} For any two random variables $X,Y$, 
\begin{align*}
    \mathbb{E}\left[|X||Y|\right] \leq \sqrt{\mathbb{E}[X^2]}\sqrt{\mathbb{E}[Y^2]}.
\end{align*}

\paragraph{CS inequality for random matrix and vector.} For any random matrix $X\in\mathbb{R}^{H\times H}$ and random vector $Y\in\mathbb{R}^H$, then
\begin{align*}
    \left\lVert \mathbb{E}\left[XY\right]\right\rVert_2 \leq  \mathbb{E}\left[\left\lVert XY\right\rVert_2 \right] \leq
    \mathbb{E}\left[\left\lVert X\right\rVert_{\text{op},2}\left\lVert Y\right\rVert_2 \right]  \leq  \mathbb{E}\left[\left\lVert X\right\rVert_{2}\left\lVert Y\right\rVert_2 \right] \leq \sqrt{\mathbb{E}\left[\left\lVert X\right\rVert_{2}^2\right]}\sqrt{\mathbb{E}\left[\left\lVert Y\right\rVert_{2}^2\right]}.
\end{align*}
The last inequality comes from the CS inequality of scalar random variables.

\paragraph{Expected norm inequality.} For any random tensor $X$, we have 
\begin{align*}
    \mathbb{E}[\left\lVert X\right\rVert_2] \leq \sqrt{\mathbb{E}[\left\lVert X\right\rVert_2^2] }.
\end{align*}
The result follows by considering $Y\coloneqq \left\lVert X\right\rVert_2$ as a single random variable, we have $\mathbb{E}[Y]\leq \sqrt{ \mathbb{E}[Y^2]}$. 

\section{Proof}
\label{appendix:proof}

\begin{restatable}{lemma}{lemmaexampleall}\label{lemma:exampleall} In the toy $N$-sample MC objective example, \begin{align*}
    \mathbb{V}\left[\hat{\nabla}_\theta^{\text{SF}}L(\theta)\right]=\Theta(N),\mathbb{V}\left[\hat{\nabla}_\theta^{\text{LSF}}L(\theta)\right]=\mathcal{O}(1/N), \mathbb{V}\left[\hat{\nabla}_\theta^{\text{PW}}L(\theta)\right]=0.
\end{align*}
\end{restatable}
\begin{proof}
We first reparameterize the random variable $X_i=\mu+\sigma\cdot\epsilon_i$ where $\epsilon_i\sim \mathcal{N}(0,1)$. With some calculations, we can show,
\begin{align*}
    \hat{\nabla}_\theta^{\text{SF}}L(\theta)=\frac{1}{N}\left(\sum_{i=1}^N \epsilon_i\right)^2 + \frac{\mu}{\sigma}\frac{1}{N}\sum_{i=1}^N \epsilon_i, \ 
    \hat{\nabla}_\theta^{\text{LSF}}L(\theta)=\frac{1}{N}\sum_{i=1}^N \epsilon_i^2  + \frac{\mu}{\sigma}\sum_{i=1}^N \epsilon_i,\ 
    \hat{\nabla}_\theta^{\text{PW}}L(\theta)&=1.
\end{align*}
It then follows that
\begin{align*}
    \mathbb{V}\left[\hat{\nabla}_\theta^{\text{SF}}L(\theta)\right] = \frac{\mu^2}{\sigma^2} N + \frac{2}{N} = \Theta(N),\ 
    \mathbb{V}\left[\hat{\nabla}_\theta^{\text{LSF}}L(\theta)\right] = \frac{\mu^2}{\sigma^2} \frac{1}{N} + \frac{2}{N} = \mathcal{O}(1/N),\ 
    \mathbb{V}\left[\hat{\nabla}_\theta^{\text{PW}}L(\theta)\right] = 0.
\end{align*}
\end{proof}

\lemmaexamplemcvariance*
\begin{proof}
See the proof for Lemma~\ref{lemma:exampleall}.
\end{proof}

\lemmaexamplemcvariancepw*
\begin{proof}
See the proof for Lemma~\ref{lemma:exampleall}.
\end{proof}

\lemmageneralmcnsampleexact*

\begin{proof}
Recall the definition of $G(\theta)$,
\begin{align*}
  G(\theta) \coloneqq  \mathbb{E}_{(X_i)_{i=1}^N} \left[f\left(\frac{\sum_{i=1}^N \phi(X_i,\theta)}{N},\theta\right)\right].
\end{align*}
The objective depends on $\theta$ in a few ways. It is straightforward to see that term (i) results from the fact that $X_i\sim p_\theta$. Another source of dependency is through the argument $\phi(X,\theta)$ and $f(\theta,\bar{\phi}_N)$. Fixing $X_i$s, taking partial gradient of $f$ with respect to $\theta$, we get from chain rule,
\begin{align*}
    \nabla f(\bar{\phi}_N,\theta) = \nabla_\theta f(\theta,\bar{\phi}_N) + \nabla_\theta \bar{\phi}_N \nabla_{\bar{\phi}_N}  f(\theta,\bar{\phi}_N).
\end{align*}
Expanding $\nabla_\theta \bar{\phi}_N$, we get the desired expression.
\end{proof}

\lemmaemamlgradient*

\begin{proof}
As discussed in the main text, with the conversion: $X_i\coloneqq \tau_i, \phi(X_i,\theta)\coloneqq R(\tau_i,g)\nabla_\theta \log p_{\theta,g}(\tau_i)$ and $f(\bar{\phi}_N,\theta)=V_g(\theta+\eta \bar{\phi}_N)$, we can cast meta-RL as a special instance of the generalized $N$-sample additive MC objective. By using the gradient of the $N$-sample additive MC objective shown in Lemma~\ref{lemma:general-mc-n-sample-exact}, we get the desired result.
\end{proof}

\coroemamlgradient*

\begin{proof}
Given $\theta_N'$, as assumed, we can construct unbiased estimates $\nabla \hat{V}_g(\theta_N')$ and $\hat{V}_g(\theta_N')$ to $\nabla V_g(\theta_N')$ and $V_g(\theta_N')$. This is equivalent to the following statement,
\begin{align*}
    \mathbb{E}\left[\nabla \hat{V}_g(\theta_N') \; \middle| \; \theta_N'\right] = \nabla V_g(\theta_N'), \mathbb{E}\left[ \hat{V}_g(\theta_N') \; \middle| \; \theta_N'\right] =  V_g(\theta_N').
\end{align*}
We now have the following
\begin{align*}
    &\mathbb{E} \left[ \hat{V}_g(\theta_N') \sum_{i=1}^N \nabla_\theta \log p_{\theta,g}(\tau_i) \left(I+\eta \frac{1}{N}\sum_{i=1}^N R(\tau_i,g) \nabla_\theta^2 \log p_{\theta,g}(\tau_i)\right) \nabla \hat{V}_g(\theta_N')  \right] \\
    &= \mathbb{E}\left[\mathbb{E} \left[ \hat{V}_g(\theta_N') \sum_{i=1}^N \nabla_\theta \log p_{\theta,g}(\tau_i) \left(I+\eta \frac{1}{N}\sum_{i=1}^N R(\tau_i,g) \nabla_\theta^2 \log p_{\theta,g}(\tau_i)\right) \nabla \hat{V}_g(\theta_N')  \; \middle| \; \theta_N' \right]\right] \\
    &= \mathbb{E}\left[ V_g(\theta_N') \sum_{i=1}^N \nabla_\theta \log p_{\theta,g}(\tau_i) \left(I+\eta \frac{1}{N}\sum_{i=1}^N R(\tau_i,g) \nabla_\theta^2 \log p_{\theta,g}(\tau_i)\right) \nabla V_g(\theta_N')  \right] = J_N(\theta,g).
\end{align*}
This shows that the finale estimate is unbiased.
\end{proof}

\lemmaexamplemcvariancelsf*
\begin{proof}
See the proof for Lemma~\ref{lemma:exampleall}.
\end{proof}

\coromamlgradient*

\begin{proof}
With the conversion $X_i\coloneqq \tau_i, \phi(X_i,\theta)\coloneqq R(\tau_i,g)\nabla_\theta \log p_{\theta,g}(\tau_i)$ and $f(\bar{\phi}_N,\theta)=V_g(\theta+\eta \bar{\phi}_N)$, we can derive the generalized LSF gradient estimate to the meta-RL objective as a special instance of Eqn~\ref{eq:general-mc-n-sample-lsf},
\begin{align*}
      \left(\eta \frac{1}{N}\sum_{i=1}^N R(\tau_i,g) u_i u_i^T\right)\nabla V_g(\theta_N')  \left(I+\eta \frac{1}{N}\sum_{i=1}^N R(\tau_i,g) \nabla_\theta^2 \log p_{\theta,g}(\tau_i)\right) \nabla V_g(\theta_N') .
\end{align*}
For the second part of the result, by replacing $V_g$ and $\nabla V_g$ by their unbiased estimate, we can show that the overall estimate has the same expectation. This is because $\mathbb{E}[\hat{V}_g(\theta_N')|\theta_N']=V_g(\theta_N')$ and $\mathbb{E}[\nabla \hat{V}_g(\theta_N')|\theta_N']=\nabla V_g(\theta_N')$, we can show the desired result via the law of total expectation as in the proof of Lemma~\ref{coro:emaml-gradient}.
\end{proof}

\coromamlgradientequiv*

\begin{proof}
We directly compute the gradient $J_\infty(\theta,g)=\nabla_\theta L(\theta,g)$ via chain rule,
\begin{align*}
    J_\infty(\theta,g) = \left(I+\eta \nabla^2 V_g(\theta)\right) \nabla V(\theta').
\end{align*}
Now, if we write $V_g(\theta)=\mathbb{E}_\tau\left[R(\tau,g)\right]$ with $\tau\sim p_{\theta,g}$. We can compute its Hessian,
\begin{align*}
    \nabla^2 V_g(\theta) &= \mathbb{E}_\tau \left[R(\tau,g)\nabla_\theta^2 \log p_{\theta,g}(\tau) + R(\tau,g) \nabla_\theta \log p_{\theta,g}(\tau) \left(\nabla_\theta \log p_{\theta,g}(\tau)\right)^T\right]\\
    &= \mathbb{E}_{(\tau_i)_{i=1}^N} \left[\frac{1}{N} \sum_{i=1}^N R(\tau_i,g)\nabla_\theta^2 \log p_{\theta,g}(\tau_i) + \frac{1}{N}\sum_{i=1}^N R(\tau_i,g) \nabla_\theta \log p_{\theta,g}(\tau_i) \left(\nabla_\theta \log p_{\theta,g}(\tau_i)\right)^T\right].
\end{align*}
Rearranging Eqn~\ref{eq:maml-gradient-equiv}, we get the desired result.

\end{proof}

\begin{restatable}{proposition}{lemmaimplication}\label{lemma:implication}  
For all $\theta,\theta'$ and $g$,
the value function satisfies the smoothness condition $\left\lVert \nabla^2 V_g(\theta)\right\rVert_2\leq L_1$. The following variances are bounded,
\begin{align*}
  &  \mathbb{V}\left[R(\tau,g)\nabla_\theta \log p_{\theta,g}(\tau)\right]\leq \sigma_1^2, \mathbb{V}\left[R(\tau,g)\nabla_\theta ^2 \log p_{\theta,g}(\tau)\right]\leq \sigma_{21}^2,\\
  &\mathbb{V}\left[R(\tau,g)\nabla_\theta \log p_{\theta,g}(\tau) \left(\nabla_\theta \log p_{\theta,g}(\tau)\right)^T \right]\leq \sigma_{22}^2,\\
  & \mathbb{E}\left[ \left\lVert R(\tau_i) \nabla_\theta \log p_{\theta,g}(\tau_i) - \mathbb{E}\left[R(\tau_i) \nabla_\theta \log p_{\theta,g}(\tau_i) \right]\right\rVert_2^3 \right] \leq \sigma_3^3.
\end{align*}
All quantities above $L_1,\sigma_1^2,\sigma_{21}^2,\sigma_{22}^2,\sigma_3^2$ can be expressed as functions of $(R,\gamma,H,G_1,G_2,D)$.
\end{restatable}

\begin{proof}
First, consider $L_1$. 
To express $L_1$ as a function of $(R,\gamma,H,G_1,G_2,D)$, note that we can derive an upper bound for all $\theta,g$ and $\tau$,
\begin{align*}
   \left\lVert R(\tau,g)\nabla_\theta^2 \log p_{\theta,g}(\tau) + R(\tau,g)\left(\nabla_\theta \log p_{\theta,g}(\tau)\right)\left(\nabla_\theta \log p_{\theta,g}(\tau)\right)^T  \right\rVert_2 \leq RG_2H\frac{1-\gamma^H}{1-\gamma}+RG_1^2H^2 \frac{1-\gamma^H}{1-\gamma}.
\end{align*}
This implies 
\begin{align*}
  \left\lVert \nabla^2 V_g(\theta) \right\rVert_2 &=  \left\lVert \mathbb{E}_{\tau}\left[R(\tau,g)\nabla_\theta^2 \log p_{\theta,g}(\tau) + R(\tau,g)\left(\nabla_\theta \log p_{\theta,g}(\tau)\right)\left(\nabla_\theta \log p_{\theta,g}(\tau)\right)^T \right]\right\rVert_2 \\
  &\leq   \mathbb{E}_{\tau}\left[\left\lVert R(\tau,g)\nabla_\theta^2 \log p_{\theta,g}(\tau) + R(\tau,g)\left(\nabla_\theta \log p_{\theta,g}(\tau)\right)\left(\nabla_\theta \log p_{\theta,g}(\tau)\right)^T \right\rVert_2 \right]\\
  &\leq RG_2H\frac{1-\gamma^H}{1-\gamma}+RG_1^2H^2 \frac{1-\gamma^H}{1-\gamma}.
\end{align*}
Then we can write 
\begin{align*}
\left\lVert \nabla V_g(\theta)-V_g(\theta') \right\rVert_2 = \left\lVert \nabla^2 V_g(\tilde{\theta}) (\theta-\theta') \right\rVert_2 \leq \left\lVert \nabla^2 V_g(\tilde{\theta}) \right\rVert_{\text{op},2}  \left\lVert \theta-\theta' \right\rVert_2 \leq \left\lVert \nabla^2 V_g(\tilde{\theta}) \right\rVert_{2}  \left\lVert \theta-\theta' \right\rVert_2 \leq L_1 \left\lVert \theta-\theta' \right\rVert_2 .
\end{align*}
Hence, we can set $L_1=RG_2H\frac{1-\gamma^H}{1-\gamma}+RG_1^2H^2 \frac{1-\gamma^H}{1-\gamma}$.

Regarding the variances, note that since all the random variables \begin{align*}
    R(\tau,g)\nabla_\theta \log p_{\theta,g}(\tau),\ R(\tau,g)\nabla_\theta^2 \log p_{\theta,g}(\tau),\ R(\tau,g)\nabla_\theta \log p_{\theta,g}(\tau)\left(\nabla_\theta \log p_{\theta,g}(\tau)\right)^T,
\end{align*} 
are bounded almost surely (the bounds are a function of $(R,\gamma,H,G_1,G_2,D)$), their variances are also bounded, and can be expressed as a function of such bounds. As one way to derive such bounds, we can upper bound each entry of the random tensor. For example, take $R(\tau,g)\nabla_\theta \log p_{\theta,g}(\tau)$ as an example, we can write $\left| \left[R(\tau,g)\nabla_\theta \log p_{\theta,g}(\tau)\right]_i\right|\leq RG_1 H\frac{1-\gamma^H}{1-\gamma}$ for all component $i$. For any random variable $X$ such that $\left|X\right|\leq C$, we have $\mathbb{V}[X]\leq C^2$ \citep{popoviciu1965certaines}. This implies 
$\mathbb{V}\left[R(\tau,g)\nabla_\theta \log p_{\theta,g}(\tau)\right]\leq D \cdot \left(RG_1 H\frac{1-\gamma^H}{1-\gamma}\right)^2$ and we can set $\sigma_1=\sqrt{D}\cdot RG_1 H\frac{1-\gamma^H}{1-\gamma}$. We refer to such bounds as the loose bounds.

Such bounds might not have an optimal dependency on $(R,\gamma,H,G_1,G_2,D)$. For example, since for all $\theta,g$ and $\tau$, we can bound 
\begin{align*}
    \left\lVert R(\tau,g)\nabla_\theta \log p_{\theta,g}(\tau)\right\rVert_2 \leq RG_1 H\frac{1-\gamma^H}{1-\gamma}.
\end{align*}
This implies that the random vector $R(\tau,g)\nabla_\theta \log p_{\theta,g}(\tau)$ has bounded norm almost surely. By definition of the vector variance, this also implies that $\sigma_1^2$ is bounded. We can bound
\begin{align*}
    \mathbb{V}\left[ R(\tau,g)\nabla_\theta \log p_{\theta,g}(\tau) \right] \leq \mathbb{E}\left[\left\lVert R(\tau,g)\nabla_\theta \log p_{\theta,g}(\tau) \right\rVert_2^2\right] \leq \left(RG_1 H\frac{1-\gamma^H}{1-\gamma}\right)^2. 
\end{align*}
We can hence set $\sigma_1=RG_1 H\frac{1-\gamma^H}{1-\gamma}$, which is an improvement over the naive approach with a factor of $\sqrt{D}$. Nevertheless, it is straightforward to derive the loose bounds for $\sigma_{21}^2,\sigma_{22}^2$ and $\sigma_3^2$ and conclude the result, though tighter bounds require more refined analysis.

\end{proof}

\begin{restatable}{proposition}{propbias}\label{prop:bias1} For all $\theta\in\mathbb{R}^D$, $\left\lVert \mathbb{E}\left[\hat{J}_{N,\text{LSF}}(\theta)\right]  - J_\infty(\theta)\right\rVert_2  \leq \mathcal{O}(1/\sqrt{N})$.
\end{restatable}

\begin{proof}

Since $\mathbb{E}\left[\hat{J}_{N,\text{LSF}}(\theta)\right]  - J_\infty(\theta) = \mathbb{E}_g\left[ \mathbb{E}\left[\hat{J}_{N,\text{LSF}}(\theta,g)\right]  - J_\infty(\theta,g) \right]$, we focus on the bias of the task-conditioned bias $\mathbb{E}\left[\hat{J}_{N,\text{LSF}}(\theta,g)\right]  - J_\infty(\theta,g)$. Henceforth we will suppress the dependency of the trajectories on the task variable, still denoting the $N$ trajectories as $(\tau_i)_{i=1}^N$. 
We write $\mathbb{E}\left[\hat{J}_{N,\text{LSF}}(\theta,g)\right]  - J_\infty(\theta,g)$ as follows,
\begin{align*}
  \mathbb{E}_{(\tau_i)_{i=1}^N}\left[\left(I+\eta \frac{1}{N}\sum_{i=1}^N R(\tau_i,g) \nabla_\theta^2 \log p_{\theta,g}(\tau_i) + \eta\frac{1}{N}\sum_{i=1}^N R(\tau_i,g) u_i u_i^T
  \right) \left(\nabla V_g(\theta_N') - \nabla V_g(\theta')\right) \right].
\end{align*}
For notational simplicity,
we define the following
\begin{align*}
    &X_N\coloneqq \eta \frac{1}{N}\sum_{i=1}^N R(\tau_i,g) \nabla_\theta^2 \log p_{\theta,g}(\tau_i) + \eta\frac{1}{N}\sum_{i=1}^N R(\tau_i,g) u_i u_i^T \\ &Y_N\coloneqq \nabla V_g(\theta_N')-\nabla V_g(\theta').
\end{align*}
Note that $X\coloneqq\mathbb{E}[X_N]=\eta\nabla_\theta^2 V_g(\theta)$. We write the above difference as 
\begin{align*}
    \mathbb{E}\left[(I+X_N) Y_N\right] = \underbrace{\mathbb{E}\left[(I+X)Y_N\right]}_{\text{part (i)}} + \underbrace{\mathbb{E}[(X_N-X)Y_N]}_{\text{part (ii)}}.
\end{align*}
To bound the norms of each term, note that we have due to the Lipschitz smoothness of the value function parameterization,
\begin{align*}
\mathbb{E}\left[\left\lVert  Y_N \right\rVert_2^2 \right] \leq L_1^2 \mathbb{E}\left[\left\lVert  \theta_N'-\theta' \right\rVert_2^2 \right] \leq L_1^2\frac{\eta^2\sigma_1^2}{N}.
\end{align*}
The Lipschitz smoothness also implies $\left\lVert X \right\rVert_2=\eta\left\lVert \nabla_\theta^2 V_g(\theta) \right\rVert_2\leq \eta L_1$. This entails a bound on part (i) of the difference,
\begin{align*}
    \left\lVert \text{part (i)} \right\rVert_2\leq  \left\lVert (I+X) \mathbb{E}[Y_N] \right\rVert_2 \leq \left\lVert I+X \right\rVert_{\text{op},2} \left\lVert  \mathbb{E}[Y_N] \right\rVert_2 \leq (1+\eta L_1) \sqrt{\mathbb{E}\left[\left\lVert  Y_N \right\rVert_2^2\right]} \leq (1+\eta L_1) L_1\frac{\eta\sigma_1}{\sqrt{N}}.
\end{align*}
We have exchanged the norms and expectation, and applied the expected norm inequality.

To bound the second part, first note that for any two random variables $X,Y$, the sum of the variance is upper bounded as: $\mathbb{V}[X+Y]\leq 2\mathbb{V}[X]+2\mathbb{V}[Y]$. This inequality extends to general random tensor $X,Y$. This implies
\begin{align*}
    \mathbb{V}\left[R(\tau_i,g)\nabla_\theta ^2 \log p_{\theta,g}(\tau_i) + R(\tau_i,g)\nabla_\theta \log p_{\theta,g}(\tau_i) \left(\nabla_\theta \log p_{\theta,g}(\tau_i)\right)^T \right]\leq \sigma_{21}^2 + \sigma_{22}^2
\end{align*}
This further implies,
\begin{align*}
    \mathbb{E}\left[\left\lVert X_N-X \right\rVert_2^2 \right] = \mathbb{V}\left[X_N\right] \leq \frac{\eta^2}{N}(\sigma_{21}^2 + \sigma_{22}^2).
\end{align*}
Before bounding the norm of part (ii), we recall that that for any matrix $X$ and vector $Y$, we have $\left\lVert XY \right\rVert_2\leq \left\lVert X\right\rVert_{\text{op},2}\left\lVert Y\right\rVert_2$, and that $\left\lVert X\right\rVert_{\text{op},2} \leq \left\lVert X\right\rVert_2$. Now, we can upper bound part (ii) as follows,
\begin{align*}
    \left\lVert \text{part (ii)}\right\rVert_2 \leq \mathbb{E}\left[\left\lVert (X_N-X)Y_N \right\rVert_2 \right] \leq \mathbb{E}\left[\left\lVert (X_N-X) \right\rVert_{\text{op},2}\left\lVert Y_N \right\rVert_2 \right] \leq \mathbb{E}\left[\left\lVert (X_N-X) \right\rVert_2\left\lVert Y_N \right\rVert_2 \right].
\end{align*}
The final RHS is upper bounded by the following due to Cauchy–Schwarz inequality, which implies
\begin{align*}
     \left\lVert \text{part (ii)}\right\rVert_2\leq ...\leq \sqrt{\mathbb{E}\left[\left\lVert (X_N-X) \right\rVert_2^2 \right]}\sqrt{\mathbb{E}\left[\left\lVert Y_N \right\rVert_2^2 \right]} \leq \frac{\eta}{\sqrt{N}}\sqrt{\sigma_{21}^2+\sigma_{22}^2} \cdot L_1 \frac{\eta\sigma_1}{\sqrt{N}}.
\end{align*}
Combining the two results above with a triangle inequality due to the vector 2-norm, we have 
\begin{align*}
   \left\lVert \mathbb{E}\left[\hat{J}_{N,\text{LSF}}(\theta,g)\right]  - J_\infty(\theta,g) \right\rVert_2 \leq (1+\eta L_1)\frac{L_1\eta\sigma_1}{\sqrt{N}}+\sqrt{\sigma_{21}^2+\sigma_{22}^2}\frac{L_1 \eta^2 \sigma_1}{N}.
\end{align*}
This induces the final result,
\begin{align*}
    \left\lVert\mathbb{E}[\hat{J}_{N,\text{LSF}}(\theta)]-J_\infty(\theta) \right\rVert_2 \leq \mathbb{E}_g \left[ \left\lVert \mathbb{E}\left[\hat{J}_{N,\text{LSF}}(\theta,g)\right]  - J_\infty(\theta,g) \right\rVert_2\right] \leq (1+\eta L_1)\frac{L_1\eta\sigma_1}{\sqrt{N}}+\sqrt{\sigma_{21}^2+\sigma_{22}^2}\frac{L_1 \eta^2 \sigma_1}{N}.
\end{align*}

\end{proof}

\begin{restatable}{proposition}{propbiassecond}\label{prop:bias2} For all $\theta\in\mathbb{R}^D$,
$
     \left\lVert J_\infty(\theta)-J_N(\theta)\right\rVert_2  \leq \mathcal{O}(1/\sqrt{N})$.
\end{restatable}
\begin{proof}

We seek to bound the difference $ \left\lVert \mathbb{E}\left[\hat{J}_{N,\text{LSF}}(\theta)\right]  - J_N(\theta)\right\rVert_2$. We first focus on bounding task-conditional gradient $ \left\lVert \mathbb{E}\left[\hat{J}_{N,\text{LSF}}(\theta,g)\right]  - J_N(\theta,g)\right\rVert_2$. To this end, recall that $u_i\coloneqq\nabla_\theta \log p_{\theta,g}(\tau_i)$, then we can write
\begin{align*}
    \mathbb{E}\left[\hat{J}_{N,\text{LSF}}(\theta,g)\right]  - J_N(\theta,g) =  \mathbb{E}_{(\tau_i)_{i=1}^N}\left[ \left(\eta\frac{1}{N}\sum_{i=1}^N R(\tau_i,g) u_i u_i^T \right)
   \nabla V_g(\theta_N') \right] - \mathbb{E}_{(\tau_i)_{i=1}^N}\left[ \left(\eta\sum_{i=1}^N u_i \right)
    V_g(\theta_N') \right].
\end{align*}
We start with Taylor expansion of $V_g(\theta_N')$ with respect to the reference point $\theta'\coloneqq\theta+\eta\mathbb{E}[R(\tau,g)\nabla \log p_{\theta,g}(\tau)]$. 
\begin{align*}
    V_g(\theta_N')=V_g(\theta') +  \left(\nabla V_g(\theta')\right)^T (\theta_N'-\theta') + 1/2 \cdot (\theta_N'-\theta')^T  \nabla^2 V_g(\theta) (\theta_N'-\theta')^T + 1/6 \cdot \nabla^3 V_g(\tilde{\theta})(\theta_N'-\theta')^3,
\end{align*}
where $\tilde{\theta}$ is a random vector between $\theta_N'$ and $\theta'$.
Here, for $A\in\mathbb{R}^{H\times H\times H}$ and $x\in\mathbb{R}^H$ we define the notation $Ax^3\coloneqq \sum_{ijk} A_{ijk}x_ix_jx_k$. Plugging in the expansion, $\mathbb{E}\left[\hat{J}_{N,\text{LSF}}(\theta,g)\right]  - J_N(\theta,g)$ evaluates to
\begin{align*}
     &= \mathbb{E}\left[\left(\eta\frac{1}{N}\sum_{i=1}^N R(\tau_i,g) u_i u_i^T \right) \nabla V_g (\theta_N')\right] - \underbrace{\mathbb{E}\left[ \left(\eta\sum_{i=1}^N u_i \right)
    V_g(\theta') \right]}_{=0}  \\
    &\ \ \ - \mathbb{E}\left[ \left(\eta\sum_{i=1}^N u_i \right)
    \left(\nabla V_g(\theta')\right)^T (\theta_N'-\theta') \right]- \mathbb{E}\left[ \left(\eta\sum_{i=1}^N u_i \right)
    1/2 \cdot (\theta_N'-\theta')^T  \nabla^2 V_g(\tilde{\theta}) (\theta_N'-\theta')^T \right] \\
    &= \underbrace{\mathbb{E}\left[\left(\eta\frac{1}{N}\sum_{i=1}^N R(\tau_i,g)u_i u_i^T\right) \left(\nabla V_g(\theta_N')-\nabla V_g(\theta')\right)\right]}_{\text{part (i)}}- \underbrace{\mathbb{E}\left[ \left(\eta\sum_{i=1}^N u_i \right)
    1/2 \cdot (\theta_N'-\theta')^T  \nabla^2 V_g(\theta') (\theta_N'-\theta')^T \right]}_{\text{part (ii)}} \\
    &\ \ \ - \underbrace{\mathbb{E}\left[ \left(\eta\sum_{i=1}^N u_i \right)
    1/6 \cdot \nabla^3 V_g(\tilde{\theta})(\theta_N'-\theta')^3 \right]}_{\text{part (iii)}}.
\end{align*}

Below, we bound each of the three parts above. For part (i), let $X_N=\eta\frac{1}{N}\sum_{i=1}^N R(\tau_i,g)u_iu_i^T$ and $Y_N=\nabla V_g(\theta_N')-\nabla V_g(\theta')$. Let $X=\mathbb{E}[X_N]$. Then,
\begin{align*}
    \mathbb{E}[\left\lVert X_N\right\rVert_2^2]=\left\lVert\mathbb{E}[X_N] \right\rVert_2^2 + \mathbb{V}[X_N].
\end{align*}
To obtain a bound of the norm on $\mathbb{E}[X_N]$, let $u=\nabla_\theta \log  p_{\theta,g}(\tau)\left(\nabla_\theta \log  p_{\theta,g}(\tau)\right)^T $ for any $\tau$, then
\begin{align*}
   \left\lVert R(\tau,g) u u^T \right\rVert_2 \leq R\frac{1-\gamma^{H}}{1-\gamma} \left\lVert uu^T \right\rVert_2 = R\frac{1-\gamma^{H}}{1-\gamma} \sqrt{\text{Trace}(uu^T uu^T)} = R\frac{1-\gamma^{H}}{1-\gamma} (u^T u) \leq H^2 G^2 R\frac{1-\gamma^{H}}{1-\gamma}.
\end{align*}
In the above we used the assumption that $\left\lVert \nabla_\theta \log \pi_\theta(a|s,g)\right\rVert_2\leq G$. This implies
\begin{align*}
    \left\lVert \mathbb{E}[X_N] \right\rVert_2 =\eta \left\lVert  \mathbb{E}[R(\tau,g)u u^T]\right\rVert_2  \leq \eta \mathbb{E}\left[\left\lVert R(\tau,g)u u^T \right\rVert_2 \right] \leq \eta H^2G^2 R \frac{1-\gamma^H}{1-\gamma}.
\end{align*}
We hence conclude
\begin{align*}
    \mathbb{E}[\left\lVert X_N \right\rVert_2^2] \leq \eta^2 H^4 G^4 R^2 \left(\frac{1-\gamma^H}{1-\gamma}\right)^2 + \frac{\eta^2 \sigma_{22}^2}{N}.
\end{align*}
Combining with the previous result on the norm bound of $\mathbb{E}[\left\lVert Y_N\right\rVert_2^2]$, we get via the CS inequality,
\begin{align*}
    \left\lVert \text{part (i)} \right\rVert_2 \leq \sqrt{\mathbb{E}[\left\lVert  X_N \right\rVert_2^2]}\sqrt{\mathbb{E}[\left\lVert  Y_N \right\rVert_2^2]} \leq \sqrt{\eta^2 H^4 G^4 R^2 \left(\frac{1-\gamma^H}{1-\gamma}\right)^2 + \frac{\eta^2 \sigma_{22}^2}{N}} \cdot L_1 \frac{\eta\sigma_1}{N}.
\end{align*}

Now, consider part (ii). For notational simplicity, we denote $x_i \coloneqq R(\tau_i,g)\nabla_\theta \log p_{\theta,g}(\tau_i)$ and $y_i\coloneqq \nabla_\theta \log p_{\theta,g}(\tau_i)$. Let $A\coloneqq \nabla^2 V_g(\theta+\eta\mu)$ be the Hessian.
\begin{align*}
    \text{part (ii)} &= \mathbb{E}\left[\frac{1}{2} \left(\frac{1}{N}\sum_i x_i - \mu\right)^T A \left(\frac{1}{N}\sum_i x_i - \mu\right) \sum_i y_i \right] \\
    &= \mathbb{E}\left[ \frac{1}{2} \frac{1}{N^2} \sum_{ijk} x_i^T A x_j y_k - \frac{1}{2}\frac{1}{N}\left(\sum_{ij} x_i^T A \mu y_j + \sum_{ij}\mu^T A x_i y_j  \right) \right].
\end{align*}
Recall that $\mathbb{E}\left[y_i\right]=0$. We can simplify the above as follows
\begin{align*}
    \mathbb{E}\left[ \frac{1}{2} \frac{1}{N^2} \left( \sum_{j\neq i} x_i^T A x_j y_i  + \sum_{j\neq i} x_i^T A x_j y_j + \sum_i x_i^T Ax_i y_i \right)- \frac{1}{2}\frac{1}{N}\left(\sum_{i} x_i^T A \mu y_i + \sum_{i}\mu^T A x_i y_i \right) \right].
\end{align*}
Let $\mu_1\coloneqq \mathbb{E}\left[x_i^T A  \mu y_i\right]$, $\mu_2 \coloneqq\mathbb{E}\left[\mu^T A x_i y_i\right] $ and $\mu_3 \coloneqq \mathbb{E}[x_i^T Ax_i y_i]$. Then 
\begin{align*}
    &\left\lVert\mathbb{E}\left[\frac{1}{N^2} \sum_{j\neq i} x_i^T Ax_j y_i \right]-\frac{1}{N}\sum_i x_i^T Ax_j y_i \right\rVert_2  = \frac{1}{N} \left\lVert \mu_1\right\rVert_2  \\
    &\left\lVert\mathbb{E}\left[\frac{1}{N^2} \sum_{j\neq i} x_i^T Ax_j y_j \right]-\frac{1}{N}\sum_i \mu^T Ax_i y_i \right\rVert_2 = \frac{1}{N} \left\lVert \mu_2\right\rVert_2.\\
    &\left\lVert \frac{1}{N^2} \sum_i x_i^T Ax_i y_i \right\rVert_2 \leq \frac{1}{N} \left\lVert \mu_3 \right\rVert_2.
\end{align*}
Overall, we can bound the norm of part (ii) by $
    \frac{1}{N} \left( \left\lVert \mu_1 \right\rVert_2+ \left\lVert \mu_2 \right\rVert_2+ \left\lVert \mu_3 \right\rVert_2\right)
$. 
We  now provide bounds to the norms of $\mu_1,\mu_2,\mu_3$ above. Take $\mu_1$ as an example,
\begin{align*}
    \left\lVert \mu_1 \right\rVert_2 \leq \mathbb{E}\left[\left\lVert x_i^T A \mu y_i \right\rVert_2\right] =  \mathbb{E}\left[\left|x_i^T A \mu\right|\left\lVert  y_i \right\rVert_2\right] \leq L_1 \mathbb{E}\left[\left|x_i^T \mu \right| \left\lVert y_i \right\rVert_2\right] \leq L_1 \mathbb{E}\left[\left\lVert x_i\right\rVert_2 \left\lVert\mu\right\rVert_2  \left\lVert y_i \right\rVert_2\right].
\end{align*}
Finally, note that we have $\left\lVert x_i \right\rVert_2 \leq HGR\frac{1-\gamma^H}{1-\gamma}$, $\left\lVert \mu \right\rVert_2 \leq HGR\frac{1-\gamma^H}{1-\gamma}$ and $\left\lVert y_i \right\rVert_2 \leq HG$, we can conclude
\begin{align*}
    \left\lVert \mu_1 \right\rVert_2\leq L_1 H^3 G_1^3 R^2 \left(\frac{1-\gamma^H}{1-\gamma}\right)^2.   
\end{align*}
We can derive similar bounds 
\begin{align*}
    \left\lVert \mu_2 \right\rVert_2\leq L_1 H^3 G_1^3 R^2 \left(\frac{1-\gamma^H}{1-\gamma}\right)^2, \left\lVert \mu_3 \right\rVert_2\leq L_1 H^3 G_1^3 R^2 \left(\frac{1-\gamma^H}{1-\gamma}\right)^2. 
\end{align*}
Overall, this implies a bound on part (ii). 
\begin{align*}
    \left\lVert \text{part (ii)}\right\rVert_2 \leq \frac{3}{N}\cdot L_1 H^3 G_1^3 R^2 \left(\frac{1-\gamma^H}{1-\gamma}\right)^2.
\end{align*}

Now, finally consider part (iii). We first note that Assumption~\ref{assum:common2} implies that $\left|\nabla^3 V_g(\theta)\right|\leq L$. This further implies that for any vector $x\in\mathbb{R}^D$,
\begin{align*}
    \nabla^3 V_g(\theta) x^3 \coloneqq \sum_{ijk}\left(\nabla^3 V_g(\theta)\right)_{ijk}x_ix_jx_k \leq \underbrace{D^3 L}_{\eqqcolon L_2} \left\lVert x \right\rVert_2^3.
\end{align*}
Recall that $\mu\coloneqq\mathbb{E}[R(\tau,g)\nabla_\theta \log p_{\theta,g}(\tau)]$ is the expected PG at $\theta$. We have the following,
\begin{align*}
    \left\lVert \text{part (iii)}\right\rVert_2 &\leq \mathbb{E}\left[\left| \nabla^3 V_g(\tilde{\theta})\left( \frac{1}{N}\sum_{i=1}^N R(\tau_i)\nabla_\theta \log p_{\theta,g}(\tau_i) - \mu \right)^3  \right| \left\lVert \sum_i \nabla_\theta \log p_{\theta,g}(\tau_i) \right\rVert_2 \right] \\
    &\leq \mathbb{E}\left[L_2 \left\lVert \left( \frac{1}{N}\sum_{i=1}^N R(\tau_i)\nabla_\theta \log p_{\theta,g}(\tau_i) - \mu \right)  \right\rVert_2^3 \left\lVert \sum_i \nabla_\theta \log p_{\theta,g}(\tau_i) \right\rVert_2 \right] 
    \\ & \leq L_2 N HG \mathbb{E}\left[\left\lVert \frac{1}{N}\sum_{i=1}^N R(\tau_i)\nabla_\theta \log p_{\theta,g}(\tau_i) - \mu \right\rVert_2^3\right] \leq L_2NHG_1\cdot \frac{1}{N^{3/2}}\sigma_3^3  = \frac{1}{\sqrt{N}}L_2HG_1\sigma_3^3.
\end{align*}
By combining the bounds in part (i)-(iii) via a triangle inequality, we arrive at the desired result.

\end{proof}

\propbiasthird*

\begin{proof}
Combining Proposition~\ref{prop:bias1} and Proposition~\ref{prop:bias2} with a triangle inequality, we get the desired result. 
\end{proof}

\begin{restatable}{proposition}{proplipschitz}\label{prop:lipschitz} The objective $J_\infty(\theta)$ is Lipschitz with constant $(1+\eta L_1)RG_1H\frac{1-\gamma^H}{1-\gamma}$.
\end{restatable}
\begin{proof}
By construction $\nabla_\theta J_\infty(\theta)=\mathbb{E}_g\left[\nabla_\theta J_\infty(\theta,g)\right]$, we can derive
\begin{align*}
    \nabla_\theta J_\infty(\theta) = \mathbb{E}_g\left[\left(I+\eta\nabla^2 V_g(\theta)\right) \nabla V_g(\theta_g')\right],
\end{align*}
where $\theta_g'\coloneqq\theta+\eta\mathbb{E}[R(\tau,g)\nabla \log p_{\theta,g}(\tau)]$. Recall that $\left\lVert \nabla^2 V_g(\theta) \right\rVert_{\text{op},2}\leq \left\lVert \nabla^2 V_g(\theta) \right\rVert_2 \leq L_1$ by the assumption and $\left\lVert \nabla V_g(\theta) \right\rVert_2 \leq RG_1H\frac{1-\gamma^H}{1-\gamma}$. We conclude
\begin{align*}
    \left\lVert \nabla_\theta J_\infty(\theta)\right\rVert_2 \leq \mathbb{E}_g\left[(1+\eta L_1)RG_1H\frac{1-\gamma^H}{1-\gamma} \right] = (1+\eta L_1)RG_1H\frac{1-\gamma^H}{1-\gamma}.
\end{align*}
\end{proof}

\propvariance*

\begin{proof}
Recall that the LSF gradient estimate is constructed with three sources of randomness: $B$ sampled task variables $g_i$, $N$ sampled trajectories per task for the inner loop $\tau_{ij}$ and $M$ sampled trajectories per task for the outer loop PG estimate $\tau_{ik}'$. The variance of $\hat{J}_{N,\text{LSF}}$ comes from these three sources of randomness. Recall that given random variable $Z_1, Z_2, Y$, the variance of $Y $ can be decomposed into three parts,
\begin{align*}
    \mathbb{V}[Y] = \underbrace{ \mathbb{E}\left[\mathbb{V}[Y|Z_1,Z_2]\right]}_{\text{first}} + \underbrace{ \mathbb{E}\left[\mathbb{V}\left[\mathbb{E}\left[Y|Z_1,Z_2\right]|Z_1\right]\right]}_{\text{second}} + \underbrace{\mathbb{V}\left[\mathbb{E}[Y|Z_1]\right]}_{\text{third}}.
\end{align*}
By definition of the variance of general random tensors, the above formula can be extended to the case where $Y$ is a random tensor. Recall the LSF gradient estimate
\begin{align*}
  \frac{1}{B} \sum_{i=1}^B \left( \left(I+\eta \sum_{j=1}^N R(\tau_{ij},g_i)\nabla_\theta^2 \log p_{\theta,g_i}(\tau_{ij}) \right) \nabla \hat{V}_{g_i}(\theta_{i,N}') + \eta\frac{1}{N}\sum_{j=1}^N R(\tau_{ij},g_i) u_{ij}u_{ij}^T \nabla \hat{V}_{g_i}(\theta_{i,N}') \right), 
\end{align*}
where $u_{ij}\coloneqq \nabla_\theta \log p_{\theta,g_i}(\tau_{ij}) $. Here, recall that 
$\hat{V}_{g_i}(\theta_{i,N}')$ is the $M$-sample estimate of outer loop PG based on $\tau_{ik}'$. Using the variance decomposition formula, we set $Y=\text{part (i)}$ and let $Z_1=(g_i)_{i=1}^B, Z_2=(\tau_{ij})_{i=1,j=1}^{B,N}$. 

\paragraph{First part.} For notational simplicity, let $A_{i} \coloneqq \left(I+\eta \frac{1}{N}\sum_{j=1}^N R(\tau_{ij},g_i)\nabla_\theta^2 \log p_{\theta,g_i}(\tau_{ij}) \right)$ and let $B_i=\eta\frac{1}{N}\sum_{j=1}^N R(\tau_{ij},g_i)u_{ij}u_{ij}^T$. Note that both $A_i$ and $B_i$ are random matrices. 
The first part of the variance is
\begin{align*}
    \mathbb{E}_{Z_1,Z_2}\left[\mathbb{V}\left[\frac{1}{B}\sum_{i=1}^B  (A_i+B_i) \hat{V}_{g_i}(\theta_{i,N}')|Z_1,Z_2\right]\right] =  \frac{1}{B}\sum_{i=1}^B \mathbb{E}_{Z_1,Z_2}\left[\left\lVert (A_i+B_i) \right\rVert_2^2 \mathbb{V}\left[  \hat{V}_{g_i}(\theta_{i,N}')|Z_1,Z_2\right]\right],
\end{align*}
where we have used the conditional independence across different $i$ indices and the fact that for any constant matrix $A$ and zero mean vector $x$: \begin{align*}
    \mathbb{V}[Ax] &\leq \mathbb{E}\left[\left\lVert Ax \right\rVert_2^2 \right]=\mathbb{E}\left[ x^T A^T Ax \right] \leq \mathbb{E}\left[ x^T x\cdot  \max_i \left|\sigma_i\left(A^T A\right)\right|  \right] \\ &\leq \mathbb{E}[x^Tx \sqrt{\sum_i \sigma_i^2(A^TA)}] = \left\lVert A^TA \right\rVert_2\mathbb{E}\left[x^T x \right] \leq \left\lVert A \right\rVert_2^2\mathbb{E}[x^T x] = \left\lVert A \right\rVert_2^2 \mathbb{V}[x].
\end{align*}
In the above, $\sigma_i(A)$ denotes the $i$-th eigenvalue of matrix $A$. We have also used the fact
that $\sqrt{\sum_i \sigma_i^2(A^TA)} = \left\lVert  A^TA \right\rVert_2 \leq \left\lVert A \right\rVert_2^2 $. Since $\hat{V}_{g_i}(\theta_{i,N}')$ is $M$-sample estimate of PG from $\theta_{i,N}'$, from previous proof, we conclude $\mathbb{V}\left[  \hat{V}_{g_i}(\theta_{i,N}')|Z_1,Z_2\right]\leq \frac{L_1^2\sigma_1^2}{M}$. To bound the norm of each $A_i+B_i$, note that
\begin{align*}
    \mathbb{E}[\left\lVert A_i+B_i \right\rVert_2^2]\leq \left\lVert  \mathbb{E}[A_i+B_i]\right\rVert_2^2 + \mathbb{V}[A_i+B_i].
\end{align*}
Now, note $\mathbb{E}[A_i+B_i]=I+\eta\nabla^2 V_g(\theta)$ whose 2-norm is bounded as $\left\lVert I+\eta\nabla^2 V_g(\theta)\right\rVert_2 \leq \left\lVert I \right\rVert_2+\left\lVert \eta \nabla^2 V_g(\theta) \right\rVert_2 \leq \sqrt{D}+\eta L_1$. Next, by recalling $\mathbb{V}[X+Y]\leq 2\mathbb{V}[X]+2\mathbb{V}[Y]$, we have
\begin{align*}
    \mathbb{V}[A_i+B_i]\leq2 \eta^2 \frac{\sigma_{21}^2+\sigma_{22}^2}{N}.
\end{align*}
Combining all previous results, we have the first part of the variance is upper bounded as follows
\begin{align*}
     \mathbb{E}_{Z_1,Z_2}\left[\mathbb{V}\left[\frac{1}{B}\sum_{i=1}^B  (A_i+B_i) \hat{V}_{g_i}(\theta_{i,N}')|Z_1,Z_2\right]\right]  \leq \left((\sqrt{D}+\eta L_1)^2 + 2\eta^2\frac{\sigma_{21}^2+\sigma_{22}^2}{N}\right) \cdot\frac{L_1^2\sigma_1^2}{M}.
\end{align*}

\paragraph{Second part.} Using notations above, we first integrate over the randomness in trajectories $\tau_{ik}'$,
\begin{align*}
    \mathbb{E}[Y|Z_1,Z_2] = \frac{1}{B}\sum_{i=1}^B \left(A_i + B_i\right) \nabla V_{g_i}(\theta_{i,N}').
\end{align*}
Using the conditional independence of $i$ given $Z_1$, we deduce
\begin{align*}
    \mathbb{V}\left[ \mathbb{E}[Y|Z_1,Z_2]|Z_1\right] = \frac{1}{B}\sum_{i=1}^B \underbrace{\mathbb{V}[\left(A_i + B_i\right) \nabla V_{g_i}(\theta_{i,N}')|g_i]}_{p_i}.
\end{align*}
Now, consider each term $p_i$ above. We can let $X_i=A_i+B_i, Y_i=\nabla V_{g_i}(\theta_{i,N}')$. We also define $X\coloneqq \mathbb{E}[X_i]$ and $Y\coloneqq \nabla V_{g_i}(\theta_i')$ where $\theta'=\theta+\eta\mathbb{E}[R(\tau,g_i)\nabla_\theta \log p_{\theta,g_i}(\tau)]$. Importantly, note that $\mathbb{E}[Y_i]\neq Y$. By using the definition of variance and Cauchy–Schwarz inequality, we have
\begin{align*}
    p_i\leq\mathbb{E}\left[\left\lVert X_i Y_i - XY \right\rVert_2^2\right] \leq   \underbrace{\mathbb{E}\left[\left\lVert X_iY_i - XY_i \right\rVert_2^2\right]}_{\text{term (a)}} +\underbrace{ \mathbb{E}\left[\left\lVert XY_i - XY \right\rVert_2^2\right]}_{\text{term (b)}} + 2\sqrt{\mathbb{E}\left[\left\lVert X_iY_i - XY_i \right\rVert_2^2\right]}\sqrt{\mathbb{E}\left[\left\lVert XY_i - XY \right\rVert_2^2\right]}.
\end{align*}
Consider term (a),
\begin{align*}
     \mathbb{E}\left[\left\lVert X_iY_i - XY_i \right\rVert_2^2\right]\leq \mathbb{E}\left[\left\lVert X_i-X \right\rVert_2^2\left\lVert Y_i \right\rVert_2^2 \right] &\leq \left(RHG_1\frac{1-\gamma^H}{1-\gamma}\right)^2 \mathbb{E}\left[\left\lVert X_i-X \right\rVert_2^2\right] \\ &\leq  \left(RHG_1\frac{1-\gamma^H}{1-\gamma}\right)^2 \frac{2\eta^2(\sigma_{21}^2+\sigma_{22}^2)}{N}.
\end{align*}

Now we consider term (b)
\begin{align*}
    \mathbb{E}\left[\left\lVert XY_i - XY_i \right\rVert_2^2\right] \leq \mathbb{E}\left[\left\lVert X \right\rVert_2^2\left\lVert Y_i-Y \right\rVert_2^2 \right] \leq  \left(1+\eta L_1\right)^2 \mathbb{E}\left[\left\lVert Y_N-Y \right\rVert_2^2 \right] \leq \left(1+\eta L_1\right)^2 \frac{L_1^2\sigma_1^2}{N},
\end{align*}
where we have applied a bound on the norm of the PG $\leq RHG_1\frac{1-\gamma^H}{1-\gamma}$ and on the Hessian operator norm $\leq L_1$ implied by the assumptions. We thus conclude the following bound on the second variance term,
\begin{align*}
   \mathbb{E}\left[ \mathbb{V}\left[\mathbb{E}\left[Y|Z_1,Z_2\right]|Z_1\right]\right] \leq \mathbb{E}\left[\frac{1}{B}\sum_{i=1}^B p_i\right] &\leq \left(RHG_1\frac{1-\gamma^H}{1-\gamma}\right)^2 \frac{2\eta^2(\sigma_{21}^2+\sigma_{22}^2)}{N} + \left(1+\eta L_1\right)^2 \frac{L_1^2\sigma_1^2}{N} + \\
    &\ \ 2\left(RHG_1\frac{1-\gamma^H}{1-\gamma}\right)(1+\eta L_1)L_1\eta\sigma_1\sqrt{2(\sigma_{21}^2+\sigma_{22}^2)}\frac{1}{N}.
\end{align*}

\paragraph{Third part.} 
Let $X_i=I+\eta \sum_{j=1}^N R(\tau_{ij},g_i)\nabla_\theta^2 \log p_{\theta,g_i}(\tau_{ij}) + \eta\frac{1}{N}\sum_{j=1}^N R(\tau_{ij},g_i) u_{ij}u_{ij}^T$ and $Y_i=\nabla V_{g_i}(\theta_{i,N}')$. Also let $Y=Y_i=\nabla V_{g_i}(\theta_i')$. We can write
\begin{align*}
    \mathbb{V}\left[\mathbb{E}[Y|Z_1]\right]  = \frac{1}{B} \mathbb{V}\left[\mathbb{E}[X_iY_i]\right] 
\end{align*}
where we have used the independence across different $i$. For clarity, the expectation is w.r.t. all randomness in $\tau_{ij}$ and $\tau_{ik}'$, whereas the variance is w.r.t. the randomness in $i$. Now, for any $i$, consider the following
\begin{align*}
    \left\lVert\mathbb{E}\left[X_iY_i\right]\right\rVert_2^2 &\leq \mathbb{E}\left[\left\lVert X_iY_i\right\rVert_2^2\right] \\
    &\leq \mathbb{E}\left[\left\lVert X_i\right\rVert_2^2 G_1^2 R^2 H^2 \left(\frac{1-\gamma^H}{1-\gamma}\right)^2\right]\\
    &= G_1^2 R^2 H^2 \left(\frac{1-\gamma^H}{1-\gamma}\right)^2\cdot \left(\mathbb{V}\left[X_i\right]+\left\lVert \mathbb{E}\left[X_i\right]\right\rVert_2^2\right)\\
    &\leq G_1^2 R^2 H^2 \left(\frac{1-\gamma^H}{1-\gamma}\right)^2\cdot \left(2\eta^2\frac{\sigma_{21}^2+\sigma_{22}^2}{N}+\left\lVert \mathbb{E}\left[X_i\right]\right\rVert_2^2\right).
\end{align*}
We need to upper bound the 2-norm of $\mathbb{E}[X_i]=I+\eta \nabla^2 V_g(\theta)$. Note that we have $\left\lVert I+\eta \nabla^2 V_g(\theta)\right\rVert_2 \leq \left\lVert I\right\rVert_2 + \left\lVert \eta \nabla^2 V_g(\theta)\right\rVert_2=\sqrt{D}+\eta L_1$. We hence have 
\begin{align*}
    \mathbb{V}\left\lVert \mathbb{E}[Y|Z_1] \right\rVert_2\leq \frac{1}{B} \mathbb{E}\left[\left\lVert \mathbb{E}[X_iY_i] \right\rVert_2^2\right] \leq \frac{1}{B} G_1^2 R^2 H^2 \left(\frac{1-\gamma^H}{1-\gamma}\right)^2\cdot \left(2\eta^2\frac{\sigma_{21}^2+\sigma_{22}^2}{N}+\left(\sqrt{D}+\eta L_1\right)^2\right).
\end{align*}

\paragraph{Combining all parts.} Combining all the three parts above, we have
\begin{align*}
    \mathbb{V}\left[\hat{J}_{N,\text{LSF}}(\theta)\right] \leq \underbrace{\mathcal{O}(1/M)}_{\text{first part}}+\underbrace{\mathcal{O}(1/N)}_{\text{second part}}+\underbrace{\mathcal{O}(1/B)}_{\text{third part}}.
\end{align*}

\end{proof}

\begin{restatable}{lemma}{lemmabiasedsgd}\label{lemma:biasedsgd} (Adapted from \citep{ajalloeian2020analysis}) Let $F:\mathbb{R}^H\mapsto\mathbb{R}$ be a $L$-Lipschitz function. Let $\hat{g}(x)$ be an estimate to $\nabla F(x)$. Its bias and variance properties are the following,
\begin{align*}
    \left\lVert \mathbb{E}\left[\hat{g}(x)\right] - \nabla F(x) \right\rVert_2^2  \leq \psi^2, \mathbb{V}\left[\hat{g}(x)\right]\leq \sigma^2,
\end{align*}
for all $x\in\mathbb{R}^H$. Now consider the recursion: $x_{t+1}=x_t +\alpha \hat{g}(x_t)$. For any $\epsilon>0$, if we choose the learning rate $\alpha=\min\{\frac{1}{L},\frac{\epsilon+\psi^2}{2L\sigma^2}\}$, then for $T=\max\{\frac{1}{\epsilon+\psi^2},\frac{\sigma^2}{\epsilon^2 + \psi^4}\}$ iterations, we have
\begin{align*}
    \min_{1\leq t\leq T} \mathbb{E}\left[\left\lVert \nabla F(x_t)\right\rVert_2^2 \right] \leq \epsilon+\psi^2.
\end{align*}
\end{restatable}
\begin{proof}
Please refer to \citep{ajalloeian2020analysis} for detailed proof and the complete result.

\end{proof}

\propmain*

\begin{proof}
We directly draw results from \citep{ajalloeian2020analysis}
where they establish convergence to stationary point using biased stochastic gradient estimates. See Lemma~\ref{lemma:biasedsgd} for a simplified version of their result which will be useful for our analysis below.

Instead of directly characterizing the convergence to a stationary point of $J_N$, we consider how fast the algorithm converges to a stationary point of $J_\infty$. Proposition~\ref{prop:lipschitz} shows that $J_\infty$ is Lipschitz-smooth with a Lipschitz constant $L$ independent of $N$, hence we can write $L=\mathcal{O}(1)$. Proposition~\ref{prop:bias2} shows that the estimate bias $\psi$ is of order $\mathcal{O}(1/\sqrt{N})$ such that
\begin{align*}
    \left\lVert J_\infty(\theta) - \mathbb{E}[\hat{J}_{N\text{LSF}}(\theta)] \right\rVert_2^2 \leq \psi^2.
\end{align*}
Proposition~\ref{prop:variance} shows that the variance $\mathbb{V}\left[\hat{J}_{N\text{LSF}}(\theta)\right]\leq \sigma^2= \mathcal{O}(1) + \mathcal{O}(1/N)$ (note that we treat $B,M$ as $\mathcal{O}(1)$ here).  Directly using results from Proposition~\ref{lemma:biasedsgd}, we obtain the following: after $T_{\text{LSF}}=\max\{\frac{1}{\epsilon^2+\mathcal{O}(1/N)},\frac{\mathcal{O}(1)+\mathcal{O}(1/N)}{\epsilon^4 + \mathcal{O}(1/N^2)}\}$ iterations,
\begin{align*}
    \min_{1\leq t\leq T_\text{LSF}} \mathbb{E}[\left\lVert \nabla_\theta J_\infty(\theta_t) \right\rVert_2^2]\leq  \epsilon^2 + \mathcal{O}(1/N).
\end{align*}
Finally, recall that Proposition~\ref{prop:bias1} upper bounds the bias between $\nabla_\theta J_\infty(\theta)$ and $J_N(\theta)$ by $\mathcal{O}(1/N)$, we obtain via a CS inequality $\left\lVert a+b \right\rVert_2^2 \leq 2\left\lVert a \right\rVert_2^2 + 2\left\lVert b \right\rVert_2^2$,
\begin{align*}
    \min_{1\leq t\leq T_\text{LSF}} \mathbb{E}[\left\lVert J_N(\theta_t) \right\rVert_2^2] \leq 
    \min_{1\leq t\leq T_\text{LSF}} 2\mathbb{E}[\left\lVert \nabla_\theta J_\infty(\theta_t) \right\rVert_2^2] + 2\mathbb{E}[\left\lVert J_N(\theta_t) - J_\infty(\theta_t) \right\rVert_2^2] 
    \leq 
    2\epsilon^2 + \mathcal{O}(1/N).
\end{align*}
By properly scaling $\epsilon$, the above is equivalent to:  after $T_{\text{LSF}}=2\cdot \max\{\frac{1}{\epsilon^2+\mathcal{O}(1/N)},\frac{\mathcal{O}(1)+\mathcal{O}(1/N)}{\epsilon^4 + \mathcal{O}(1/N^2)}\}$ iterations,
\begin{align*}
    \min_{1\leq t\leq T_\text{LSF}} \mathbb{E}[\left\lVert J_N(\theta_t) \right\rVert_2^2] \leq
    \epsilon^2 + \mathcal{O}(1/N).
\end{align*}
\end{proof} 

\section{Additional experiments}
\label{appendix:exp}

\subsection{Toy 1-D Optimization Problem} The variable $X$ comes from the distribution $p_\theta=\mathcal{N}(\theta,1^2)$. This implies a natural reparameterization, $X=\theta+\zeta,\zeta\sim\mathcal{N}(0,1)$, which we use for the PW gradient estimate. 

\paragraph{Bias-variance trade-off.} To generate Fig~\ref{fig:exp}(a), we compute the MSE of different gradient estimates against the true gradient, evaluated at the initial parameter of the algorithm. In particular, given the ground truth gradient $g$, we generate $M=100$ gradient estimates $\hat{g}_i$ for each type, and compute
\begin{align*}
   \frac{1}{M}\sum_{i=1}^M \left\lVert \hat{g}_i - g\right\rVert_2^2,
\end{align*}
as an estimate to the MSE. For LSF, the bias is computed as  
\begin{align*}
   \frac{1}{M}\sum_{i=1}^M \left\lVert \hat{g}_i - \bar{g}\right\rVert_2^2,
\end{align*}
where $\bar{g} = \frac{1}{M}\sum_{i=1}^M \hat{g}_i$ is an estimate of the expected gradient estimate. SF and PW gradient estimates have zero variance, so that all their MSE consists of variance. Finally, we use the average over $1000$ PW gradient estimates as $g$ (because PW has low variance and zero bias, we expect the approximation to be reasonably accurate).

\paragraph{Optimization.} To generate Fig~\ref{fig:exp}(b), for each type of gradient estimate, at each iteration $1\leq t\leq T$, we construct an average gradient estimate $\bar{g}=\frac{1}{B}\sum_{i=1}^B \hat{g}_i$ where $\hat{g}_i$ is an one-sample gradient estimate of the $N$-sample MC objective. The parameter is updated with $\hat{g}$ at each iteration, with Adam optimizer \citep{kingma2014adam} and learning rate $0.1$.

At $T=100$, we record the objective $L(\theta_T)$ for each type of gradient estimate. We repeat the same experiment $100$ times, and compare $\text{mean}\pm\text{std}$ averaged over such repeated trials.

\subsection{High-dimensional Meta-RL Problems} 

The following specifies details of generating Fig~\ref{fig:exp}(c)-(d). 

\paragraph{Environments.} The environments of the meta-RL experiments are based on MuJoCo \citep{todorov2012}, and imported directly from the open source projects of \citep{rothfuss2018promp}. These are robotics control tasks where the states $s_t$ are sensory inputs and actions $a_t$ are controls applied to the robots. Across all three tasks we considered, the task $g$ corresponds to different directions in which the robot should aim to run. See \citep{rothfuss2018promp} for further details.

\paragraph{Trust region outer loop optimization.} After obtaining the gradient estimate $\hat{J}$, Algorithm 1 suggests that we update $\theta_{t+1}=\theta_t+\alpha \hat{J}$. In practice, we adopt trust region policy optimization \citep{schulman2015}, which enforces a trust region constraint between $\theta_t$ and $\theta_{t+1}$ when updating the parameter. See the open sourced code base for hyper-parameter settings of the TRPO optimizer.

\paragraph{Hyper-parameters of algorithms.} We use a batch of $B=20$ tasks per iteration, $N=M=20$ trajectories per task for both inner loop adaptation and outer loop rollouts for PG estimates. Each trajectory is truncated at $H=100$ steps. We adapt only one step throughout the experiments. Please refer to the open sourced code base for other default hyper-parameters whose details we omit here.

\paragraph{Important implementation details.}
Though all algorithms are based on the open source project of \citep{rothfuss2018promp}, it is worth noting a number of important modifications that we make to ensure that the implementation adheres to our theoretical setups as much as possible.

The unbiased generalized SF gradient estimate is very closely related to the gradient estimate used in E-MAML algorithm. In fact, when implemented exactly, the E-MAML algorithm utilizes the SF gradient estimate  $\hat{J}_{N,\text{SF}}^{(i)}(\theta,g)+\hat{J}_{N,\text{SF}}^{(ii)}(\theta,g)$ defined in Eqn~\ref{eq:unbiased-emaml}. However, the code base in \citep{rothfuss2018promp} effectively uses the following gradient estimate,
\begin{align*}
    \frac{1}{N}\hat{J}_{N,\text{SF}}^{(i)}(\theta,g)+\hat{J}_{N,\text{SF}}^{(ii)}(\theta,g).
\end{align*}
The factor $1/N$ biases the overall estimate, yet reduces the variance introduced by $\hat{J}_{N,\text{SF}}^{(i)}(\theta,g)$. With this, we see that the practical implementation of E-MAML already introduces bias for variance reduction, albeit in a more opaque way. To obtain results of the SF gradient estimate in Fig~\ref{fig:exp}, we remove the $1/N$ factor and use the unbiased SF gradient estimate.

\subsection{Practical (prior) implementations of generalized LSF gradient estimate}

To implement the generalized LSF estimate, we need to construct unbiased estimate to value function Hessian $\nabla^2 V_g(\theta)$ evaluated at the initial policy parameter $\theta$. On a high level, one first constructs an estimate $\hat{H}_N\approx\nabla^2 V_g(\theta) $ and then computes the meta-RL gradient estimate as
\begin{align*}
    \hat{H}_N \nabla V_g(\theta_N'),
\end{align*}
where $\theta_N'$ is the (random) updated parameter. A number of prior work discusses on how to construct unbiased estimates \citep{foerster2018dice,mao2019baseline,farquhar2019loaded} or biased estimates \citep{rothfuss2018promp,tang2020taylor} to the value function Hessian. A major desiderata is that such estimates should lead to variance reduction compared to the naive "trajectory-based" estimate. Concretely, the "trajectory-based" Hessian estimate is (derived from Eqn~\ref{eq:biased-emaml})
\begin{align}
    \frac{1}{N}\sum_{i=1}^N R(\tau_i,g) \nabla_\theta^2 \log p_{\theta,g}(\tau_i) +  \frac{1}{N}\sum_{i=1}^N R(\tau_i,g) \left(\nabla_\theta \log p_{\theta,g}(\tau_i)\right)\left(\nabla_\theta \log p_{\theta,g}(\tau_i)\right)^T.\label{eq:naive-hessian}
\end{align}
Arguably, the variance of the above estimate could be further improved by exploiting the Markov property of trajectories $\tau_i$. Taking computing PG estimate as an analogy, when computing "trajectory-based" PG estimate $R(\tau_i)\nabla_\theta \log p_{\theta,g}(\tau_i)$, we can instead use its "stepwise-based" variant $\sum_{t=0}^{H-1} \gamma^t \hat{Q}_t \nabla_\theta \log \pi_\theta(a_t|s_t,g)$, where $\hat{Q}_t$s are unbiased estimates to Q-functions. The "stepwise-based" estimate usually has much lower variance than the "trajectory-based" estimate, because it is constructed based on the Markov structure of the trajectory. Constructing variance-reduced estimates for the Hessian is more complicated, but is better understood through the lens of off-policy evaluation. We refer readers to \citep{tang2021unifying} for further details.

In our experiments, we always use such "stepwise-based" PG and Hessian estimates when computing the meta-RL gradient estimates. Specifically, we use DiCE \citep{foerster2018dice} to compute the LSF gradient estimate, which can be interpreted as building an unbiased variant of Eqn~\ref{eq:naive-hessian} with variance reduction via the Markov structure of the trajectory. Please refer to the code base of \citep{rothfuss2018promp} for further implementation details.

\paragraph{Full results.} See Fig~\ref{fig:exp-appendix} for full results on the high-dimensional meta-RL problems. Overall, the LSF estimate achieves significant performance gains over the SF estimate.

\begin{figure}[t]
    \centering
    \subfigure[Meta-RL HalfCheetah ]{\includegraphics[keepaspectratio,width=.3\textwidth]{plots/halfcheetahranddirec_nsamples_big.png}}
    \subfigure[Meta-RL Ant ]{\includegraphics[keepaspectratio,width=.3\textwidth]{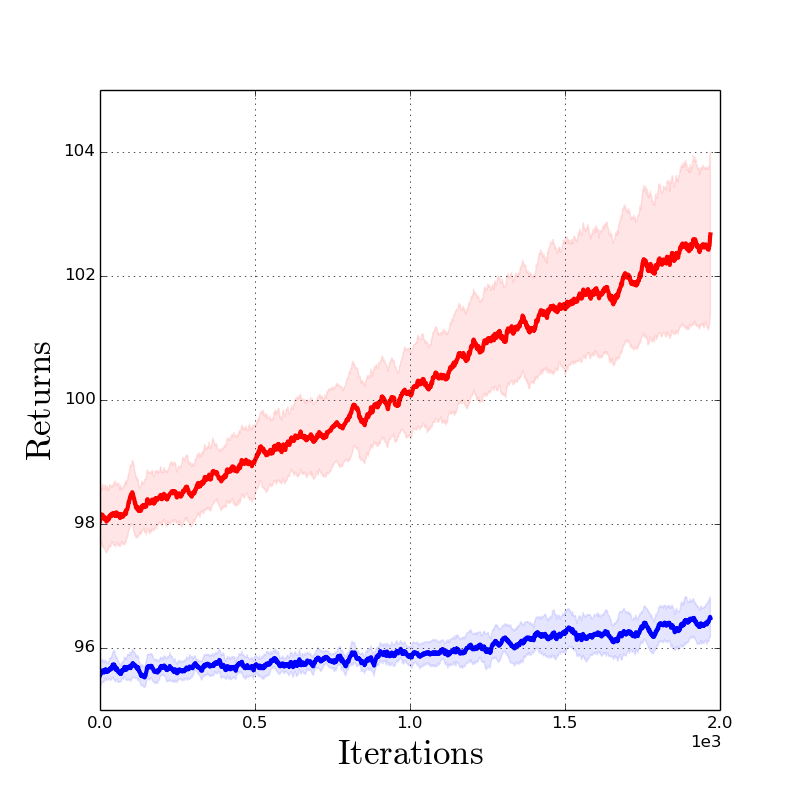}}
    \subfigure[Meta-RL Walker2D ]{\includegraphics[keepaspectratio,width=.3\textwidth]{plots/walker2dranddirec_nsamples_big.png}}
    \caption{ Full results of high-dimensional meta-RL problems: LSF outperforms SF consistently across all tasks. Each curve averages over $5$ runs.}
    \label{fig:exp-appendix}
\end{figure}

\end{document}